\documentclass{article}

\usepackage[nonatbib,preprint]{neurips_2025}

\usepackage{silence}
\ActivateWarningFilters[pdftoc]

\usepackage[utf8]{inputenc}
\usepackage[T1]{fontenc}

\usepackage{fdk-math}
\usepackage[nomarginkeys]{fdk-cite}
\usepackage{fdk-typesetting}
\usepackage{enumitem}
\usepackage{hyperref}
\usepackage{url}
\usepackage{booktabs}
\usepackage{amsfonts}
\usepackage{dsfont}
\usepackage{tikz}
\usepackage{nicefrac}
\usepackage{microtype}
\usepackage{xcolor}

\usepackage{thmtools}
\usepackage{thm-restate}

\usepackage{outlines}

\usepackage{xparse}

\newtheorem{proposition}{Proposition}[section]
\newtheorem{theorem}[proposition]{Theorem}
\newtheorem{assumption}[proposition]{Assumption}

\newtheorem{lemma}[proposition]{Lemma}

\newtheorem{definition}[proposition]{Definition}
\newtheorem{problem}[proposition]{Problem}

\usepackage{tcolorbox}

\newcommand{\mtop}{{\scriptscriptstyle\top}}

\AtEveryBibitem{%
  \clearlist{address}
  \clearfield{eprint}
  \clearfield{isbn}
  \clearfield{issn}
  \clearlist{location}
  \clearfield{month}
  \clearfield{series}
 \ifentrytype{book}{}{%
  \clearlist{publisher}
  \clearname{editor}
 }
}

\newcommand{\ind}[1]{\mathds{1}_{#1}}

\newcommand{\dd}[1]{\frac{\partial}{\partial{#1}}}
\renewcommand{\E}{\mathrm{E}}

\NewDocumentCommand\minfo{sm}{%
  \IfBooleanTF{#1}%
    {\textcolor{black}{Info: #2}}
    {\marginpar{\scriptsize\color{black}Info: #2}}
}
\NewDocumentCommand\mwarn{sm}{%
  \IfBooleanTF{#1}%
    {\textcolor{orange}{Warn: #2}}
    {\marginpar{\scriptsize\color{orange}Warn: #2}}
}

\newcommand{\legendDashes}{%
\tikz[baseline]{%
  \draw[dash pattern=on 6pt off 4pt on 6pt,line width=0.4ex] (0,.5ex)--++(0.65,0);%
}\!%
}
\newcommand{\legendDashesOrange}{%
\tikz[baseline]{%
  \draw[dash pattern=on 6pt off 4pt on 6pt, draw=orange, line width=0.4ex] (0,.5ex)--++(0.65,0);%
}\!%
}

\usepackage{ifthen}
\newcommand{\asym}[1]{%
  \ifthenelse{\equal{#1}{}}{%
    \sim%
  }{%
    \mathrel{\raisebox{-.075em}{%
      \smash{$\stackrel{%
        \raisebox{-.075em}{$\scriptscriptstyle #1$}%
      }{%
        \sim%
      }$}%
    }}%
  }%
}

\title{
Scaling Laws for Gradient Descent and Sign Descent\\
for Linear Bigram Models under Zipf's Law
}

\author{%
  Frederik Kunstner%
  \\
  \texttt{frederik.kunstner@inria.fr} \\
  \And
  Francis Bach%
  \\
  \texttt{francis.bach@inria.fr} \\
}

\usepackage{blindtext}
\usepackage{lipsum}

\newcommand{\fixme}[1]{\textcolor{red}{{#1}}}

\renewcommand{\mwarn}[1]{}

\begin{document}
\maketitle

\newcommand{\fillerP}[1]{\textcolor{red}{\lipsum[#1]}}
\newcommand{\fillerS}[2]{\textcolor{red}{\lipsum[#1][#2]}}

\begin{abstract}
Recent works have highlighted optimization difficulties 
faced by gradient descent in training the first and last layers of transformer-based language models, 
which are overcome by optimizers such as Adam.
These works suggest that the difficulty is linked to the heavy-tailed distribution of words in text data, 
where the frequency of the $k$th most frequent word $\pi_k$ is proportional to $1/k$, following Zipf's law.
To better understand the impact of the data distribution on training performance,
we study a linear bigram model for next-token prediction
when the tokens follow a power law~$\pi_k \propto 1/k^\alpha$
parameterized by the exponent~$\alpha > 0$.
We derive optimization scaling laws for deterministic gradient descent 
and sign descent as a proxy for Adam as a function of the exponent~$\alpha$.
Existing theoretical investigations in scaling laws
assume that the eigenvalues of the data decay as a power law with exponent~$\alpha > 1$.
This assumption effectively makes the problem ``finite dimensional''
as most of the loss comes from a few of the largest eigencomponents.
In comparison, we show that the problem is more difficult when the data have heavier tails.
The case~$\alpha = 1$ as found in text data is ``worst-case'' for gradient descent,
in that the number of iterations required to reach a small relative error scales almost linearly with dimension.
While the performance of sign descent also depends on the dimension,
for Zipf-distributed data the number of iterations scales only with the square-root of the dimension, 
leading to a large improvement for large vocabularies.
\end{abstract}

\begin{outline}[tightitemize]

\section{Introduction}

Recent works have shown that one of the primary benefits of Adam~\citep{kingma2015adam}
in training transformed-based language models~\citep{vaswani2017attention} 
lies in how it handles the first and last layers~\citep{zhang2025adamminiusefewerlearning,zhao2025deconstructing}. 
For language models, the input and output dimensions correspond to distinct words in the vocabulary, 
where the~$k$th most frequent word has frequency~$\pi_k \propto 1/k$ following Zipf's law \citep{piantadosi2014zipf}.
\citet{kunstner2024heavytailed} provide evidence that this heavy-tailed distribution
leads to optimization difficulties for gradient descent that Adam is able to overcome.
They argue that Zipf's law is ``worst-case'' in that it combines a large imbalance in frequencies, 
while decaying slowly enough that most samples come from the~tail.

Our objective is to formalize this empirical observation,
and to describe the impact 
of the heavy-tailedness of the data distribution
on the convergence of gradient descent (GD) and sign descent (SD) 
as a proxy for Adam~\citep{tieleman2012rmsprop,bernstein2018signsgd,balles2020geometrysigndescent,chen2023symbolic}.
We consider a linear bigram model for next-token prediction trained with the square loss,
where the token frequencies $\pi_k$ follow a power law $\pi_k \propto 1/k^\alpha$ with exponent $\alpha > 0$.
While this problem would be solved directly rather than with iterative methods, 
it is a good starting point for the theoretical investigation of optimization dynamics.
Despite its apparent simplicity, this model already reproduces the observation
that GD performs poorly on Zipf-distributed data (see \cref{fig:motivation}).
The behavior of gradient and sign descent are also not well described by current results, 
see~\cref{sec:related-work}.

Our approach is inspired by the line of work on theoretical scaling laws,
or asymptotic convergence as the dimensionality grows~\citep[e.g.,][]{caponnetto2007source,advani2020highdim,berthier2020tight,bahri2021explaining,cui2021generalization,maloney2022solvable,paquette2024phases}.
Instead of analyzing the generalization error of online gradient descent 
as the dimension of the model and sample size grow, 
we study the convergence rate of GD as the dimension and the number of iterations grow.
Spectral assumptions on the eigenvalues of the Hessian following a power-law are common in the literature,
as they correspond to assumptions on the covariance of the features.
But these works focus on power-laws that are not ``too'' heavy-tailed, $1/k^\alpha$ with $\alpha > 1$, 
which lead to sublinear convergence rates independent of dimension.
In contrast, we focus on the case $\alpha \leq 1$ where it becomes impossible to make progress with a fixed number of steps as the dimensionality grows.
We show that it is possible to make progress by finding the right scaling of the number of iterations with the dimension.

\newcommand{\shrinkspace}{\vspace{-0.5em}}
\begin{figure}[t]
\includegraphics[width=\textwidth]{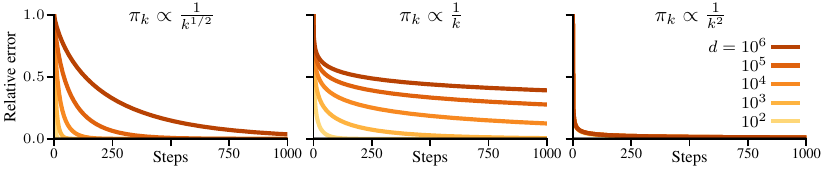}%
\shrinkspace
\caption{\textbf{Gradient descent (GD) scales badly with vocabulary size when the data is Zipfian.}
Relative error on a linear bigram problem with squared loss trained with GD
with vocabulary size~$d$ when the word class frequencies follow $\pi_k \propto 1/k^\alpha$.
For $\alpha \leq 1$ (left, middle) the performance degrades with vocabulary size, 
with worst scaling for Zipf-distributed data ($\alpha = 1$).
When the frequencies have lighter tails ($\alpha = 2$, right)
GD works well for all vocabulary sizes.
Our objective is to derive scaling laws explaining this behavior.
}
\label{fig:motivation}
\end{figure}

\subsection{Contributions}
\begin{enumerate}[left=0.5em]%
  \item We propose a simplified model of the data frequencies 
  that allows for tractable closed-form dynamics on the linear bigram problem
  for both GD and SD.
  We show experimentally that the scaling under 
  that model matches empirical performance~(\cref{fig:actual-vs-predicted-rates}).
  \item 
  We derive scaling laws for GD and SD
  under this simplified model as a function of~$\alpha > 0$.
  These results cover the challenging case of
  power-laws decreasing as slow or slower than Zipf's law~($\alpha \leq 1$)
  that is often ignored in asymptotic analyses.
  This setting leads to a qualitatively different result, 
  requiring the number of iterations to grow with $d$.
  \item 
  We show that sign descent with a well-selected step-size scales better with dimension 
  than GD for Zipf-distributed data.
  This result confirms the benefits of SD
  and preconditioning-like interventions 
  to mitigate the poor performance of GD.
  However, this result is specific to the regime $\alpha \leq 1$, 
  as sign descent exhibits worse scaling otherwise,
  showing that which algorithm performs better depends on properties of the data.
\end{enumerate}

\subsection{Overview of the results}
Given matrices
$\mX$ and $\mY \in \{0,1\}^{n\times d}$ 
containing the one-hot encodings of $n$ pairs of tokens from a vocabulary of $d$ possible words, 
we fit a linear bigram model
by minimizing the loss
\aligns{
  \Loss_d(\mW) = 
  \frac{1}{2n} \norm{\mX\mW - \mY}{}_F^2,
  \quad \text{ where }
  \quad \mW \in \R^{d\times d}, \quad
  \norm{\mX}_F^2 = \Tr(\mX^\mtop\mX).
} 
We assume that the distribution of the tokens
and the conditional distribution of the next tokens
follow a power law $1/k^\alpha$ with exponent $\alpha$,
formalized later in \cref{ass:conditional-distribution}.
Our main result is as follows. %

\begin{theorem}[Informal]
\label{thm:informal-gd}
Consider the linear bigram model when the dimensionality $d$ is large.
The number of iterations $t$ required to reach $\varepsilon$ relative accuracy 
with \emph{gradient descent} scales as follows.
\aligns{
  &\text{If } \alpha < 1,
  &
  t &\approx d^\alpha \log(1/\varepsilon),
  &
  \text{if } \alpha &= 1,
  &
  t &\approx d^{1-\varepsilon},
  &
  \text{ and if } \alpha &> 1,
  &
  t &\approx \paren{{1}/{\varepsilon}}^{\frac{\alpha}{\alpha-1}}\!.
\intertext{For \emph{sign descent}, 
there is a constant step-size $\eta$ depending on $d$ and $t$ such that, after $t$ steps,}
  &\text{if } \alpha < {1}/{2},
  &
  t &\approx \paren{{1}/{\varepsilon}}^{\frac{1}{2(1-2\alpha)}},
  \quad \, &\text{if } \alpha &= {1}/{2},
  &
  t &\approx {d}^{\frac{1-\varepsilon}{2}},
  &
  \quad  \, \text{ and if } \alpha &> {1}/{2},
  &
  t &\approx {{d}/{\varepsilon}}^{1/2}.
}
By relative accuracy, we mean that 
$\Loss_d(t) - \Loss_d^* = \varepsilon(\Loss_d(0) - \Loss_d^*)$, 
where $\Loss_d(t)$ is the loss after $t$ steps and $\Loss_d^*$ is the minimum loss.
By $t \approx f(d,\varepsilon)$, 
we mean that for some function $t(d, \varepsilon) = \tilde\Theta(f(d,\varepsilon))$
where $\tilde\Theta$ ignores terms in $\log\!\log(1/\varepsilon)$ 
and constants that depend on $\alpha$ but not on $\varepsilon$ or $d$,
we have
\aligns{
  \lim_{\varepsilon \to 0} \lim_{d\to\infty} \frac{\Loss_d(t(d,\alpha))-\Loss_d^*}{\Loss_d(0)-\Loss_d^*} \frac{1}{\varepsilon} = 1.
}
\end{theorem}

\begin{figure}[t]
\includegraphics{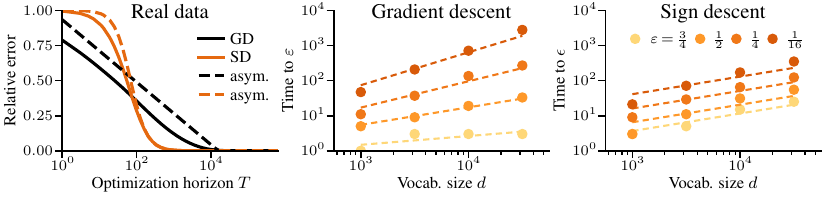}%
\shrinkspace
\caption[]{\textbf{Our scaling predicts the behavior of gradient descent and sign descent on real data.}
Left: the convergence of gradient descent (GD) and sign descent (SD)
is close to our asymptotic prediction (\legendDashes, \legendDashesOrange)
on a bigram model with~$32$k tokens on OpenWebText,
although not exactly due to the finite dimension 
and our simplified model of the frequencies in \cref{ass:conditional-distribution}.
Middle/Right:
as $d$ grows, the number of iterations required to reach $\varepsilon$ relative error matches our predictions,
showing that SD scales better with dimension for small $\varepsilon$.
We show results on real data (dots) against the scaling of~$cd^{1-\varepsilon}$ for GD and~\smash{$c d{}^{1/2}$} for SD (dashes)
where $c$ is fit to the data.
}
\label{fig:actual-vs-predicted-rates}
\end{figure}
While we do recover the traditional scaling of $(1/\varepsilon)^{-p}$ for some~$p$
for gradient descent in the case~$\alpha > 1$, our scaling laws do not all follow this functional form.
For $\alpha \leq 1$, our model recover a linear rate scaling with $d^\alpha$.
But for Zipf-distributed data ($\alpha = 1$),
the number of iterations required to reach $\varepsilon$ relative error scales with $d^{1-\varepsilon}$. 
If only a rough solution is required ($\varepsilon$ close to 1) GD scales mildly with dimension, 
but obtaining a good model ($\varepsilon$ close to 0) scales almost linearly with dimension.
SD instead scales as~\smash{$d^{1/2}$}, which is a significant improvement in large dimension,
giving a~$100$-times speedup for vocabulary sizes commonly used in practice.
We confirm these predictions experimentally 
for finite dimensional models using real data on OpenWebText, shown in \cref{fig:actual-vs-predicted-rates}.

\subsection{Related work}
\label{sec:related-work}

\textbf{Convergence of Adam and sign descent.}
Sign descent is a common proxy for Adam
as it captures the main property, that the updates are uniform across coordinates
~\citep{bernstein2018signsgd,balles2020geometrysigndescent,chen2023symbolic}.
This ``scale-freeness'' has been shown to reduce the dependence on the condition number of the problem~\citep{zhuang2022scalefree}.
However, this result does not imply SD outperforms GD, 
as known results for sign-like methods depend on the dimension $d$ instead of the condition number~\citep[e.g.,][]{safaryan2021ssd,das2024towards,liu2025adagradanisotropic}.
In the bigram problem with Zipf-distributed data, the dimension grows faster than the condition number, leading to worse guarantees for sign descent.
Instead of worst-case bounds, we rely on more fine-grained assumptions on the data and asymptotic equivalence to obtain the results of \cref{thm:informal-gd}.
We compare our results to existing rates in \cref{apx:related-work}.

\textbf{SDE approximations of sign methods.}
Scaling laws have been derived for online sign-like algorithms through 
stochastic differential equations~\citep{ma2021qualitative,malladi2022sde,xiao2024risk,compagnoni2025adaptive}.
The focus of these works is on the scaling of the step-size with batch size 
and the asymptotic stationary distribution of the algorithm which controls the generalization error.
As noise is not necessary to reproduce the performance gap between GD and Adam~\citep{kunstner2023noise},
we instead focus on the impact of heavy-tailed data on the deterministic dynamics.

\textbf{Scaling laws and asymptotic results.}
Empirical scaling laws have been developed to extrapolate the performance of deep networks at scale 
and how to balance compute across model and data sizes~\citep{rosenfeld2020constructive,kaplan2020scalinglaws,hoffman2022computeoptimal}.
Many works have contributed to the theoretical understanding of this scaling behavior 
through high dimensional analyses and random matrix theory~\citep{advani2020highdim,bahri2021explaining,maloney2022solvable,bordelon2024dynamical,lin2024scaling,paquette2024phases},
or using classical source/capacity conditions from learning theory~\citep{caponnetto2007source,berthier2020tight,cui2021generalization}, see \citet{velikanov2024tight} for the use of source/capacity conditions 
in the context of optimization.
These works assume the eigenvalues decay as a power law with exponent $\alpha > 1$.
The resulting scalings are consistent 
with the observation that training dynamics converge to a well-defined limit 
as width or depth increases~\citep{yang2021mup,bordelon2024depth,noci2024superconsistency},
but does not describe the regime $\alpha \leq 1$, which includes Zipf's law.
This regime might be more relevant when considering scaling the vocabulary size,
as in the work of \citet{gowda2020optimalvocab,tao2024scalingvocab}.
While they hypothesize that larger vocabularies might lead to worse performance due to overfitting, 
as larger vocabularies implies fewer examples per word,
we show that larger vocabulary size introduces difficulties 
in getting the \emph{training error} down.
Closest to our work is perhaps the blog post of \citet{bulatov2023harmonic}, 
who show that the loss under GD 
should approximately behave as $-\log(t/d)$ 
on a problem matching our setting with $\alpha = 1$.
Our work provides a formal justification for this scaling.

\section{Problem setup}
\label{sec:simplified_model}

In this section, we present the problem setting, 
the modeling assumptions we introduce to make the problem tractable, 
and the approach we use to derive our results.
We start from a convex quadratic in reduced form, 
$f(\vx) = \frac{1}{2}(\vx-\vx^*)^\mtop \mA(\vx-\vx^*)$,
where the eigenvalues/vectors pairs of~$\mA$~are~\smash{$(\lambda_i, \vv_i) \in \R \times \R{}^d$ for~$i =1,\ldots,d$}. 
The loss can be expressed in terms 
of the distance to the solution along each eigenvector,~$\delta_i(\vx) = \lin{\vv_i, \vx-\vx^*}$,
as~\smash{$f(\vx) = \frac{1}{2}\sum\!{}_{i=1}^d \lambda_i \delta_i(\vx){}^2$}.
The dynamics of GD with step-size~$\eta$,~$\vx_{t+1} 
= \vx_t-\eta\mA(\vx_t-\vx^*)$,
are also available in closed-form,
\alignn{
  \nonumber~\\[-1.75em]
  \!\!
  f(\vx_t) 
  =  
  \frac{1}{2}\paren{(1-\eta\mA)^{t}\paren{\vx_0-\vx^*}}^\mtop\!
  \mA 
  {(1-\eta\mA)^{t}\paren{\vx_0-\vx^*}}
  = \frac{1}{2}\sum_{i=1}^{{d}}
  \lambda_i (1-\eta\lambda_i)^{2t} \delta_i(\vx_0)^2\!.
  \label{eq:dynamics}
}
The specific quadratic problem we focus on is the 
linear bigram model with the square loss.

\newcommand{\bigramproblem}{Bigram Problem \ref{prob:bigram}}
\begin{problem}[Linear bigram model with square loss]
\label{prob:bigram}
Let~$\vx_i, \vy_i \in \{0,1\}^d$ be~$n$ samples 
representing one-hot encodings from~$d$ classes (or tokens),
with their concatenation $\mX, \mY \in \{0,1\}^{n\times d}$.
We fit a linear model with weights $\mW \in \R^{d \times d}$ using the square loss,
\aligns{
  \Loss_d(\mW) = \frac{1}{2n}\norm{\mX\mW - \mY}_F^2.
}
We define $\pi_k$ and $\pi_{k\cond j}$
as the frequencies and conditional frequency statistics of the data,
\aligns{
  \pi_k
  \coloneqq 
  \frac{1}{n} \sum_{\smash{i=1}}^n \ind{x_i = k},
  &&
  \pi_{k \cond j}
  \coloneqq 
  \frac{\sum\!{}_{i=1}^n \ind{y_i=k}\ind{x_i=j}}{\sum\!{}_{i=1}^n \ind{x_i = j}},
  \quad 
  (\text{with the convention } \nicefrac{0}{0}=0)
  \quad 
  \forall j, k \in [d].\\[-2em]
}
\end{problem}
The eigenvalues and distances to the solution 
are directly related to the frequency statistics.
\begin{proposition}
\label{prop:eigenvalues}
The eigenvalues and distances to the solution
of \cref{prob:bigram} initialized at~$\mW\,{=}\,0$~are 
\aligns{
  \smash{
  \lambda_{ij} = \pi_i
  \quad 
  \text{ and }
  \quad 
  \delta_{ij}(0) = \pi_{j\cond i}
  \quad \text{ for } i, j \in [d].
  }\\[-2em]
}
\end{proposition}
\begin{proof}
The optimization problem separates into $d$ independent $d$-dimensional subproblems,
\aligns{
  \textstyle
  \Loss_d(\mW) = \frac{1}{2n}\norm{\mX\mW-\mY}_F^2 
  = 
  \sum_{j=1}^d
  \frac{1}{2n}\norm{\mX\vw^j-\vy^j}_2^2,
}
where $\vw^j, \vy^j \in\R^d$ are the $j$th columns of $\mW, \mY$.
Each subproblem has the same Hessian given by \smash{$\mX^{\mtop} \mX/n = \Diag([\pi_1, \ldots, \pi_d])$},
so the eigenvalues are the frequencies, each with multiplicity $d$.
From $\vw^j = 0$, the distance to the solution is the magnitude of the solution $\vs^j$
of the normal equations,
\aligns{
  \textstyle
  \vs^j = \paren{\mX^\mtop \mX}^{-1} \mX^\mtop \vy^j
  = \bmat{
    \pi_{1\cond j},
    \ldots,
    \pi_{d\cond j}
  }^\mtop,
}
as $\mX$ and $\vy^j$ are one-hot, 
$(\mX^\mtop \vy^j)_k = \sum_{i=1}^n \ind{x_i = k}\ind{y_i = j}$
and $(\mX^\mtop\mX)_{kk} = \sum\!{}_{i=1}^n \ind{x_i = k}$.
\end{proof}

\subsection{Modeling assumptions}

Getting an interpretable form of the convergence of \cref{eq:dynamics}
requires assumptions on the values of~$\lambda_i$ and~$\delta_i$. 
Assuming $\mu \leq \lambda_i \leq L$
leads to the typical rates in smooth (strongly-)convex optimization, 
\aligns{
  \Loss(\vw_t) - \Loss^* \leq \frac{L \sum_{i=1}^d \delta_i(\vw_0)^2}{t},
  &&
  \Loss(\vw_t) - \Loss^* \leq \paren{1-\frac{\mu}{L}}^t 
  (\Loss(\vw_0) - \Loss^*),
}
where $\Loss^*$ is the minimum loss~\citep[see, e.g.,][]{nesterov2018intro}.
While valid, these worst-case bounds are too coarse to capture the richness of the behavior of GD
and become vacuous in high-dimensions.
We compare our results to classical rates in \cref{apx:related-work}.
To obtain more fine-grained results, we assume that the distributions
of the frequencies $\pi_k$ and conditional frequencies $\pi_{k\cond j}$ follow power laws. 
\begin{assumption}[Heavy-tailed data]
\label{ass:conditional-distribution}
We assume that the frequencies and conditional frequencies 
follow a frequency-rank power law with exponent $\alpha > 0$.
That is, assuming the frequencies are sorted~($\pi_k \geq \pi_{k+1}$) 
and defining the sorting permutations $\rho_j$ 
such that~$\pi_{\rho_j(k+1) \cond j} \geq \pi_{\rho_j(k)\cond j}$,
\aligns{
  \pi_k \propto \frac{1}{k^\alpha} 
  \quad \text{ and } \quad 
  \pi_{\rho_j(k) \cond j} \propto \frac{1}{k^\alpha},
  \quad \text{ for all } j, k, 
}
where by $\pi_k \propto \nicefrac{1}{k^\alpha}$ we mean the the distribution is normalized, 
$\pi_k = \nicefrac{1}{zk^\alpha}$ where~\smash{$z = \sum\!{}_{k=1}^d \nicefrac{1}{k^\alpha}$}.
\end{assumption}
This assumption may appear strong, 
as it would be satisfied for example if the words were sampled i.i.d. 
with frequencies $\pi_1, \ldots, \pi_d$
as \smash{$\pi_{k\cond j}=\pi_k$}.
But it does not require that all conditional distributions be the same, 
and the most likely next-token after word $j$ can depend on~$j$.
This assumption merely asks that, once sorted, 
the next-word frequencies also follow a power law with the same exponent.
Some distributions might deviate from this trend 
if a token can only logically be followed by specific tokens, 
or if the word being conditioned on is rare 
and our dataset is relatively small.\footnote{%
Even with i.i.d. data from a power law $\pi_k \propto 1/k$, 
accurately estimating the frequency of rare next-tokens 
takes a large number of samples.
With a vocabulary size of~$d=3 \cdot 10^{4}\!$,
common for large language models,
we expect to see only one
example of the pair $(x=d, y=d)$ every~$10^9$~tokens.}
While we do not expect the assumption to be exactly satisfied in practice, 
it appears to be a reasonable high-level approximation of real-world data,
as shown in \cref{fig:assumption-is-reasonable}
in comparison to the empirical distributions on OpenWebText, 
and leads to accurate predictions 
as shown in \cref{fig:actual-vs-predicted-rates}.
\begin{figure}[t]
\includegraphics[width=\textwidth]{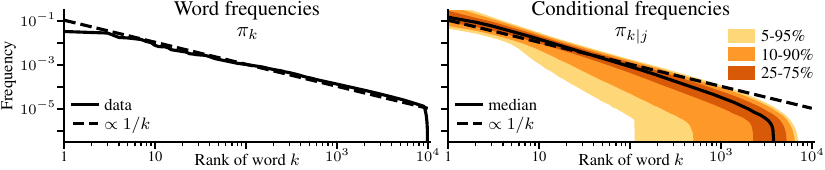}%
\shrinkspace
\caption[]{\textbf{Token frequencies and conditional frequencies approximately follow Zipf's law.}
The approximation of \cref{ass:conditional-distribution} (\legendDashes)
provides a reasonable approximation of 
the frequencies~(left) and conditional frequencies~(right) on text data,
computed on OpenWebText for a vocabulary of~$10^4$ words.
For a word~$k$, the right plot shows the median and quantiles of the distribution~$\pi_{k\cond j}$ for~$j \in [d]$.
}
\label{fig:assumption-is-reasonable}
\end{figure}
This form of spectral assumption is similar to the source/capacity conditions~\citep{caponnetto2007source}, 
see \citet{velikanov2024tight} for an account in optimization.
We compare \cref{ass:conditional-distribution} to related settings in \cref{apx:related-work}.

\subsection{Strategy for the analysis}
Our goal is to derive scaling laws for the loss of \cref{prob:bigram} in $d$ dimensions
after $t$ steps,~$\Loss_d(t)$, as~$d \to \infty$.
Such scaling laws can be interpreted as approximating the convergence rate for large $d$, 
or serve as a guide on how to scale the hyperparameters of the optimizer as we increase the vocabulary size.
Formally, we compute the asymptotic limit of the rate $r(t)$ 
at which the relative loss decreases, 
\aligns{
  \Loss_d(t)-\Loss_d^* 
  \asym{d} r(t) \paren{\Loss_d(0)-\Loss_d^*},
  \quad \text{ where $\asym{d}$ is notation for } \quad 
  \lim_{d\to\infty} \frac{\Loss_d(t)-\Loss_d^*}{\Loss_d(0)-\Loss_d^*}
  = r(t),
}
Works on scaling laws typically model the absolute value of the loss.
This approach degenerates when the loss at initialization vanishes or diverges as $d \to \infty$
which happens when $\alpha \leq 1$.
Considering the relative decrease circumvents the issue,
as also noted by \citet{bulatov2023harmonic,tao2024scalingvocab}.

Another potential degeneracy is the scaling of time. 
If the problem becomes more difficult as $d$ grows, it might be impossible to make progress in finite time.
To take a concrete example, suppose that~$\Loss_d^* = 0$
and~$\Loss_d(t) = r_d(t) \Loss_d(0)$ with~$r_d(t) = (1-1/d)^t$.
If we take the limit as~$d \to \infty$ for a fixed~$t$, 
we obtain~$\lim_{d\to\infty} (1-\nicefrac{1}{d})^t = 1$.
The rate no longer depends on $t$, and we cannot make progress unless~$t$ grows with~$d$. 
If we instead introduce a rescaled time variable~$\tau$ and scale~$t_d(\tau) = \tau d$,
we recover a linear rate in the rescaled time $\tau$
as~\smash{$(1-\nicefrac{1}{d})^{\tau d} \asym{d} e^{-\tau}$}.
This is the same issue encountered in random matrix theory, 
where the dimensions of the matrix are taken to grow jointly with a fixed ratio
to avoid degenerate solutions~\citep{potters2020firstcourse}.
It can be verified that $t_d(\tau) = \tau d$ is the ``right'' scaling, 
as the limit $r_d(t_d(\tau))$ degenerates otherwise.
Using $f(x) \ll g(x)$ for~$\lim_{x\to\infty} \nicefrac{f(x)}{g(x)} = 0$,
we have $r_d(t_d(\tau)) \asym{d} 1$ if $t_d(\tau) \ll d$
and $r_d(t_d(\tau)) \asym{d} 0$ if $t_d(\tau) \gg d$;
we either make no progress or solve the problem instantly.
Our results are derived by taking 
the finite dimensional rate $r_d(t)$
with a scaling $t_d$ such that the asymptotic rate $r(\tau)$ 
is well defined in terms of the rescaled time $\tau$, 
\alignn{
  \label{eq:asymptotic-rate}
  r(\tau) 
  \coloneqq
  \lim_{d\to\infty} r_d(t_d(\tau))
  = \lim_{d\to\infty} \frac{\Loss_d(t_d(\tau)) - \Loss_d^*}{\Loss_d(0) - \Loss_d^*}.
  \\[-2.5em]
  \nonumber
}
\section{Scaling laws for gradient descent}
We are now ready to move on to the main results for the scaling laws of gradient descent.

\begin{figure}[t]
\centering
\includegraphics[width=1.0\textwidth]{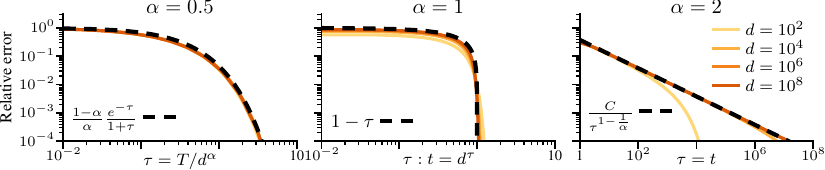}%
\shrinkspace
\caption[]{\textbf{Scaling of gradient descent on power-law data with exponent~$\alpha$ (\cref{thm:gradient-descent}).}
The dynamics of gradient descent on the linear bigram model with data satisfying \cref{ass:conditional-distribution}
converge to our scaling law (\legendDashes, \cref{thm:gradient-descent}) as $d$ grows.
Achieving a relative error $\varepsilon$ requires scaling 
the iteration budget $T$ with $d^\alpha$ for $\alpha < 1$, 
$T$ with $d^{1-\varepsilon}$ for~$\alpha = 1$,
and no scaling for~$\alpha > 1$.
}
\label{fig:gd-rates}
\end{figure}
\begin{restatable}[Scaling for gradient descent]{theorem}{thmgd}
\label{thm:gradient-descent}
On the bigram problem (Prob. \ref{prob:bigram}) with distributions 
following a power law with exponent $\alpha > 0$
(\cref{ass:conditional-distribution}),
gradient descent with a step-size $1/\pi_1$,
with time scaling $t_d(\tau)$
has the following asymptotic convergence rate (\cref{eq:asymptotic-rate}).
\aligns{
  \text{If } \alpha < 1,
  &&
  t_d(\tau) &= \tfrac{1}{2}\tau d^\alpha,
  &
  r(\tau)
  &=
  \frac{1-\alpha}{\alpha}E_{\frac{1}{\alpha}}(\tau)
  \asym{\tau}
  \frac{1-\alpha}{\alpha} \frac{e^{-\tau}}{\tau+1},
  \\
  \text{if } \alpha = 1,
  &&
  t_d(\tau) &= \tfrac{1}{2}d^\tau,
  &
  r(\tau)
  &=
  1 - \tau
  \quad \quad 
  \text{ where } \tau \in [0,1],
  \\
  \text{if } \alpha > 1,
  &&
  t_d(\tau) &= \tau,
  &
  r(\tau)
  &\asym{\tau}
  \frac{B\paren{1-\frac{1}{\alpha}, 1+2t}}{\alpha \zeta(\alpha)}
  \asym{\tau}
  C
  \frac{1}{\tau^{1-\frac{1}{\alpha}}},
}
where 
$\Gamma$ is the Gamma function, 
$E_p$ is the generalized exponential integral, 
$B$ is the Beta function, 
and $\zeta$ is the zeta function~\citep[][
\href{https://dlmf.nist.gov/5.2}{\S5.2},
\href{https://dlmf.nist.gov/5.12}{\S5.12}
\href{https://dlmf.nist.gov/8.19}{\S8.19}
\href{https://dlmf.nist.gov/25.2}{\S25.2}%
]{NIST:DLMF},
and $C = \smash{\nicefrac{\Gamma\paren{1-\frac{1}{\alpha}}}{\alpha\zeta(\alpha)}}$.
\end{restatable}
\vspace{-1em}
\begin{proof}
We sketch the proof for $\alpha = 1$ and leave the remaining cases to \cref{apx:gd}.
Under \cref{eq:dynamics} and \cref{ass:conditional-distribution}
the dynamics of the normalized loss $r_d(t) = \nicefrac{\Loss_d(t) -\Loss_d^*}{\Loss_d(0)-\Loss_d^*}$
reduce to
\aligns{
  \Loss_d(t) = 
  \bigg( \frac{1}{\sum_{k=1}^d k^{-\alpha}}
  \sum_{k=1}^d k^{-\alpha} \paren{1-k^{-\alpha}}^t \bigg)\Loss_d(0),
  \quad \text{ so } \quad 
  r_d(t) 
  = \frac{1}{H_{d,\alpha}}\sum_{k=1}^d k^{-\alpha} \paren{1-k^{-\alpha}}^t,
}
where $H_{d,\alpha} = \sum_{k=1}^d k^{-\alpha}$.
To simplify the analysis, we use the integral form of the sum
as we can use Laplace's method to estimate its behavior for large $d$,
see \cref{apx:gd} for a formal justification;
\aligns{
  \text{ For $\alpha = 1$}, \quad 
  r_d(t)
  \approx
  I_d(t) 
  &= \frac{1}{H_{d,1}} \int_1^d
  k^{-1} \paren{1-k^{-1}}^t \dif{k}
  =
  \frac{1}{H_{d,1}} \log(d) \int_0^1 \paren{1-d^{-z}}^{t} \dif{z},
}
after the change of variable $k = d^z$ or $z = \log(k)/\log(d)$.
As the normalizer $H_{d,1} \asym{d} \log(d)$,
we only need to consider the limit of the integral.
Taking $d \to \infty$ with $t$ fixed, 
the integral converges to $1$ and we make no progress, regardless of $t$. To make progress, $t$ needs to scale as $t = d^\tau$ for~$\tau \in [0,1]$,
\aligns{
  I_d(d^\tau) =
  \frac{\log(d)}{H_{d,\alpha}}
  \int_0^1 \bigg({1-\frac{d^{\tau-z}}{d^\tau}\bigg)}^{\smash{d^\tau}} \dif{z}.
}
For a fixed $\tau$ and as $d \to \infty$, the integrand converges to $0$ if $z < \tau$ and $1$ if $z > \tau$.
As it is bounded by a constant,
we can exchange limits and integrals by
the dominated convergence theorem,
obtain
\aligns{
  \lim_{d\to\infty}
  \int_0^1 
  \paren{1-\frac{d^{\tau-z}}{d^\tau}}^{\smash{d^\tau}}
  \dif{z}
  =
  \int_0^\tau 0 \dif{z}
  +
  \int_\tau^1 1 \dif{z}
  = 1-\tau.
  \tag*{\qedhere}
}
\end{proof}
The results highlight the need for a different scaling as a function of $\alpha$.
The number of iterations needs to scale with dimension 
if the data decays as slow as Zipf's law or slower ($\alpha \leq 1$)
whereas it is not necessary for lighter-tailed data ($\alpha > 1$). 
We show in \cref{fig:gd-rates} 
that the optimization dynamics in finite dimension 
on data satisfying \cref{ass:conditional-distribution}
converge to the asymptotic rates of \cref{thm:gradient-descent}
and are accurate even for common vocabulary sizes.

\section{Scaling laws for sign descent}

We now move to the case of SD.
Before going into the results, we need to address two issues.
First, the sign descent update is not linear.
We thus need to establish an alternative to the closed form solution of GD in \cref{eq:dynamics}, but for SD.
Second, SD does not converge with a fixed step-size. 
We thus need to scale step-size as a function of the iteration budget and dimension.

\subsection{Sign descent dynamics}
If run with a constant step-size, 
the update of sign descent
with a step-size of $\eta$ is 
\aligns{
  \mW_{t+1} = \mW_t - \eta \sign(\nabla \Loss(\mW_t)).
}
As the Hessian of \cref{prob:bigram} is diagonal, 
the update applies independently to each eigencomponent.
Letting $\delta_{ij}(t)$ be the distance along the $(i,j)$th eigenvector at step $t$, 
\aligns{
  \delta_{ij}(t+1)
  = \delta_{ij}(t) - \eta \sign(\delta_{ij}(t)).
}
The difficulty in the analysis 
comes from the fact that $\abs{\delta_{ij}(t)}$ does not converge to $0$.
Instead, $\abs{\delta_{ij}(t)}$ will oscillate 
between some $c \in (0, \eta)$ and $c-\eta$,
unless $t = \abs{\delta_{ij}(t)}/\eta$ is an integer.
To simplify the analysis, 
we make the following assumption, 
essentially ``averaging'' the oscillations to $\eta/2$.
\begin{figure}[t]
\includegraphics[width=\textwidth]{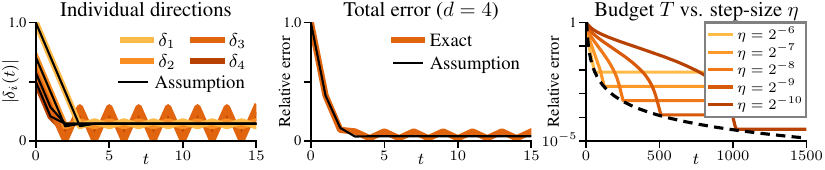}%
\shrinkspace
\caption[]{\textbf{Illustration of our modeling assumption for sign descent (\cref{ass:sign-dynamics}).}
Left: instead of modeling the oscillations of sign descent, we treat the oscillatory phase as constant.
Middle: The effect on the total error.
Right: Because SD eventually oscillates, the step-size needs to depend on the iteration budget $T$ 
to achieve best performance after $T$ steps (the envelope \legendDashes).
}
\label{fig:sign-assumption}
\end{figure}
\begin{assumption}
\label{ass:sign-dynamics}
We assume that sign descent with step-size $\eta$ follows the dynamics
\aligns{
  \abs{\delta_{ij}(t)}
  \coloneq   
  \left\{\begin{array}{ll}
    \abs{\delta_{ij}(0)} - t\eta & \text{ if } \abs{\delta_{ij}(t-1)} - \eta \geq 0,
    \\
    \eta/2  & \text{ otherwise}.
  \end{array}\right.
}
\end{assumption}
Under this assumption, 
the distances decrease while $t < \abs{\delta_{ij}(0)}/\eta$
then go to $\eta/2$ to model the oscillatory regime,\smash{\footnotemark}
as illustrated in \cref{fig:sign-assumption}.
\footnotetext{%
We could instead model the switch between the decreasing and oscillatory phase as
$\max\paren{\abs{\delta_{ij}(0)} - t\eta, \eta/2}$.
But under this model the transition occurs at $\abs{\delta_{ij}(0)} = \eta(t+1/2)$
instead of $\abs{\delta_{ij}(0)} = \eta t$.
We chose the formulation in \cref{ass:sign-dynamics} to not carry this $1/2$ term
as the difference is small for large~$t$.
Neither form captures the fact that a direction might reach exactly~0, 
after which the oscillatory phase would then not happen,
but only  a small number of directions can reach~0 if $\abs{\delta_{ij}(0)} \propto 1/j^\alpha$,
and their impact vanishes as~$d$ grows.
}
Using this assumption, we have the following dynamics.

\begin{proposition}
\label{prop:loss-of-sd}
If the conditional distribution %
follows a power law with exponent $\alpha$ as in \cref{ass:conditional-distribution},
the dynamics of sign descent with step-size $\eta$ in \cref{ass:sign-dynamics}
lead to the loss 
\aligns{
  \Loss_d(t,\eta) 
  \coloneqq \sum_{i=1}^d \sum_{j=1}^d \lambda_{ij} \delta_{ij}(t)^2
  = \sum_{k=1}^{k_*} \paren{\delta_k(0) - t\eta}^2
  + \sum_{k=k_*+1}^{d} \paren{\frac{\eta}{2}}^2
  \quad \text{ where } \quad
  \delta_k(0) = \pi_k,
}
and $k_*$ is the number of directions in the decreasing regime.
\end{proposition}
\begin{proof}
By \cref{prop:eigenvalues}, $\lambda_{ij} = \pi_i$ does not depend on $j$.
By \cref{ass:conditional-distribution}, there is a permutation $\rho_i$ such that $\delta_{i\rho_i(j)}(0) = \pi_j$.
As a result, the dynamics of \smash{$\delta_{i\rho_i(j)}(t)$} do not depend on $i$.
Writing~\smash{$\delta_{j}(t) = \delta_{i,\rho_i(j)}(t)$} for any $i$ and using that \smash{$\sum\!{}_{i=1}^d \pi_i=1$},
\aligns{
  \sum_{i=1}^d \sum_{j=1}^d \lambda_{ij} \delta_{ij}(t)^2
  =
  \sum_{i=1}^d \pi_i \sum_{j=1}^d \delta_{ij}(t)^2
  = \sum_{i=1}^d \pi_i \sum_{j=1}^d \delta_{j}(t)^2
  = \sum_{j=1}^d \delta_{j}(t)^2.
}
We then split the sum depending on whether 
$\abs{\delta_{k}(t)}$ is decreasing or oscillating.
\end{proof}

\subsection{Scaling of the step-size}
As SD with a fixed step-size eventually enters an oscillatory regime,
the loss we converge to as $t$ grows depends on $\eta$.
To describe the performance achievable after tuning $\eta$ for a given budget~$T$, 
we need to estimate how $\eta$ scales with $T$ and $d$. 
This effect is illustrated in \cref{fig:sign-assumption} (right). %
We use capital~$T$ to emphasize that we are modeling the loss at the end of a training run of $T$ steps
with a fixed step-size which depends on $T$.
Getting the exact form of $\eta_* = \min_{\eta} \Loss_d(T,\eta)$ is out of reach, 
but we establish bounds on the optimal step-size.

\begin{proposition}%
\label{prop:step-size-range-for-sign}
The step-size $\eta_*$ that $\Loss_d(T,\eta)$ in \cref{prop:loss-of-sd}
given $T$ and $d$,  %
satisfies
\aligns{
  \frac{\delta_d(0)}{T} \leq \eta_* \leq \frac{\delta_1(0)}{T}.
}
\end{proposition}
\begin{proof}
If $\eta \leq \delta_d(0)/T$, 
all directions are still in the decreasing regime of \cref{ass:sign-dynamics} at time~$T$.
As long as~$T\eta < \delta_d(0)$, increasing the step-size leads to more progress.
Similarly, if~$T\eta \geq \delta_1(0)$, all directions are in the oscillatory regime, 
and reducing the step-size reduces the oscillations.
\end{proof}

As our initial distances follow a power law, 
$\delta_k(0) = \pi_k = \frac{1}{z k^\alpha}$ where $z = \sum\!{}_{k=1}^d k^{-\alpha}$, 
\cref{prop:step-size-range-for-sign} suggests an alternative parameterization of the step-size as 
\aligns{
  \eta(\phi) = \frac{1}{zT\phi^\alpha}
  \quad  \text{ with }  \quad 
  1 \leq \phi \leq d,
}
where $\phi$ controls how many directions are still decreasing.
We now define the following scaling of $\phi$.
\begin{figure}[t]
\includegraphics[width=\textwidth]{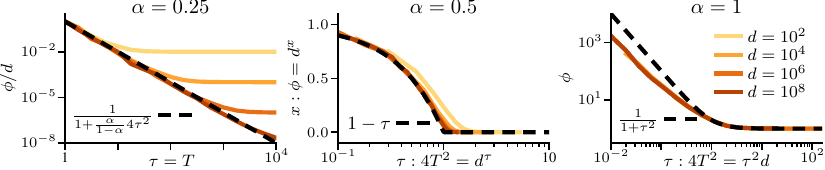}%
\shrinkspace
\caption[]{%
\textbf{Convergence of the best step-size for sign descent to the scaling in \cref{def:sign-scaling}.}
The optimal step-size for $T$ steps of sign descent converge to our scaling (\legendDashes) given in \cref{def:sign-scaling}~(for~$\tau > 1$ in the case of $\alpha = 1$).
Computed by grid search on the linear bigram model with data satisfy \cref{ass:conditional-distribution}.
}
\label{fig:sign-stepsize-convergence}
\end{figure}

\begin{definition}
\label{def:sign-scaling}
We define the following scalings 
as a function of the dimension $d$ and rescaled time $\tau$
\aligns{
  \text{if } \alpha &< 1/2, &
  T_d(\tau) &= \tau,
  & \phi_d(\tau) &= 
  \left\{\begin{array}{ll}
    d & \text{ if } \tau^2 \leq {\nicefrac{1-c_1}{4c_2}},
    \\
    d\paren{c_1 + 4c_2 \tau^2}{}^{-1} & \text{ otherwise},
  \end{array}
  \right.
  \\
  \text{if } \alpha &= 1/2, &
  T_d(\tau) &= \tfrac{1}{2} d^{\frac{1}{2}\tau},
  & \phi_d(\tau) &= d^{1-\tau}, 
  \quad \text{ where } \tau \in [0,1],
  \\
  \text{if } \alpha &> 1/2, &
  T_d(\tau) &= \tfrac{1}{2}\tau \sqrt{d},
  & 
  \phi_d(\tau) 
  &= \left\{\begin{array}{ll}
    \textstyle
    \hphantom{(} {1+\nicefrac{1}{\tau^2}} & \text{ if } \tau^2 < (2^\alpha-1)^{-1}
    \text{ and } \alpha < 1, \\
    \textstyle
    \paren{1+\nicefrac{1}{\tau^2}}{}^{1/\alpha} & \text{ otherwise},
  \end{array}\right.
}
where $c_1 = 1-\frac{1}{2\alpha}$, $c_2 = \frac{\alpha}{\alpha-1}$.
\end{definition}
While those scalings need not be optimal, %
they match the empirical behavior of the best step-size computed by grid-search, 
as shown in \cref{fig:sign-stepsize-convergence}.
For $\alpha > 1/2$, the step-size is only accurate for~\smash{$\tau^2 \geq 1/(2^\alpha-1)$} or~$\tau \geq 1$ for~$\alpha = 1$.
We justify those estimates in \cref{apx:sign}.

\subsection{Asymptotic behavior}
Using the scalings for $T$ and $\phi$ in \cref{def:sign-scaling}, 
we define the asymptotic rate of sign descent as
\alignn{
  \label{eq:asymptotic-rate-sign}
  r(\tau) = \lim_{d\to\infty} \frac{\Loss_d(T_d(\tau), \phi_d(\tau)) - \Loss_d^*}{\Loss_d(0)-\Loss_d^*}.
  \\[-1.75em]
  \nonumber
}

\begin{theorem}[Scaling for sign descent]
\label{thm:sign-rate}
Given scalings for $T$ and $\phi$ in \cref{def:sign-scaling},
the asymptotic convergence rate of sign descent (\cref{eq:asymptotic-rate-sign}) is, 
with $c_1 = 1-\frac{1}{2\alpha}$, $c_2 = \frac{\alpha}{\alpha-1}$,
\aligns{
  \text{if } \alpha &< 1/2, 
  &
  T_d(\tau) &= \tau,
  &r(\tau)
  &=
  \left\{
    \begin{array}{ll}
      \! 
      2\alpha c_2 
      & \!\! \text{if }\tau^2 \leq {\frac{1-c_1}{4c_2}}
      \\
      \! 
      \frac{
        (c_1 + c_2 4\tau^2)^{2\alpha}
      }{
        4\tau^2 
      }
      & \!\! \text{otherwise}
    \end{array}
  \right.\!\!
  \asym{\tau} 
  \frac{c_2^{2\alpha}}{(2\tau)^{2-4\alpha}},
  \\
  \text{if } \alpha &= 1/2, 
  &
  T_d(\tau) &= \tfrac{1}{2} d^{\frac{1}{2}\tau},
  & r(\tau) &=  1-\tau ,
  \quad \text{ where } \tau \in [0,1],
  \\
  \text{if } \alpha &> 1/2, 
  &
  T_d(\tau) &= \tfrac{1}{2}\tau \sqrt{d},
  &  r(\tau) &\asym{\tau} \frac{1}{1+\zeta(2\alpha) \tau^2}.
  \\[-1.75em]
  \nonumber
}
\end{theorem}
We leave the proofs in \cref{apx:sign}.
The results also show different forms of scaling depending on $\alpha$, 
with a threshold at $\alpha = 1/2$ instead of $1$.
However, the scaling in dimension is flipped compared to GD.
SD needs $t$ to scale with $d$ when $\alpha$ is large, 
which is the regime where GD can make progress with finite $t$.
However, for the case of Zipf-distributed data~($\alpha=1$), 
SD only needs a scaling in $d^{1/2}$ compared to the $d^{1-\varepsilon}$ scaling of GD, 
showing that it achieves better performance for $\varepsilon<1/2$.
We show in \cref{fig:sign-loss-convergence} 
that the asymptotic rates of \cref{thm:gradient-descent} are accurate even for finite $d$.

\begin{figure}[t]
\includegraphics[width=\textwidth]{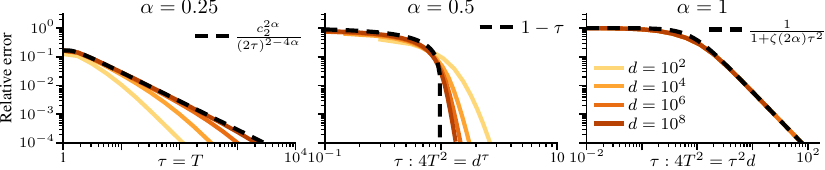}%
\shrinkspace
\caption[]{%
\textbf{Scaling of sign descent on power-law data with exponent $\alpha$ (\cref{thm:sign-rate}).}
The dynamics of sign descent on the linear bigram model 
with data satisfying \cref{ass:conditional-distribution}
converge to our scaling law (\legendDashes) as $d$ grows, 
as described in \cref{thm:sign-rate}.
Achieving a relative error $\varepsilon$
requires no scaling for~$\alpha < 1/2$,
scaling $t$ with \smash{$d^{(1-\varepsilon)/2}$} for $\alpha = 1/2$,
and $t$ with \smash{$d^{1/2}$} for $\alpha > 1/2$.
}
\label{fig:sign-loss-convergence}
\end{figure}

\section{Conclusion}

We have presented scaling laws for gradient descent (GD) and sign descent (SD)
on the linear bigram model as a function of the power law exponent $\alpha$ of the word frequencies.
Rather than hide the dimension dependence in problem specific constants, 
we consider the scaling of running time and dimension as the problem grows in size
to get precise estimates of the scaling.
Our results highlight the benefit of SD
and the need to address ill-conditioning to improve the performance of~GD.

Our results show that the typical neural scaling law $(1/\epsilon)^p$ for some $p$ is specific to the regime $\alpha > 1$. 
This regime may accurately describe the dynamics as we increase width or depth and the training dynamics converge,
but it might miss a large dimension dependence as we scale the vocabulary size.
The scaling we obtain for $\alpha <1$ and $\alpha = 1$ have a different functional form and highlight the dependency on dimension.
For GD on Zipf-distributed data, the scaling of $d^{1-\varepsilon}$ 
shows a non-trivial interplay between the desired error $\varepsilon$ and the dimension.
Our results suggests that longer training might be required when scaling the vocabulary size.
Algorithms that target this dimension dependence, for example by estimating word frequencies~\citep{li2022frequencyaware}, would be an interesting next step.

Our approach however has limitations.
We do not cover the online case, for which the analysis should be extendable using existings tools.
The addition of momentum for sign descent would be more complex but particulary interesting to dampen oscillations.
Handling more complex models such as bilinear models~\citep{mikolov2013word2vec} or the cross-entropy loss would be interesting, 
but it is not obvious how to extend the analysis without a closed-form solution for the training dynamics.
Finally, getting finite-dimensional results by tracking a correction term for finite $d$ would be enlightening, 
as the convergence to the asymptotic regime can sometimes be slow, especially in the case $\alpha = 1$.
\end{outline}

\clearpage

\begin{ack}
This work has received support from the French government, managed by the National Research Agency, 
under the France 2030 program with the reference ``PR[AI]RIE-PSAI'' (ANR-23-IACL-0008).
Frederik Kunstner is supported by a Marie Sklodowska-Curie Fellowship 
from the European Union's Horizon Europe Research and Innovation program under Grant Agreement No. 101210427.
\end{ack}

\section*{References}
\AtNextBibliography{\normalsize}%
\printbibliography[heading=none]%

\clearpage
\appendix
\part*{Supplementary Material}

\newcommand{\optgap}{\Loss_d(t) -\Loss_d^*}
\newcommand{\optgapStart}{\Loss_d(0) -\Loss_d^*}
\newcommand{\switch}[1]{\left\{
  \begin{array}{ll}
    #1
  \end{array}
  \right.
}
The supplementary material is organized as follows.
\begin{itemize}[left=1em]
  \item \cref{apx:experimental_details} gives experimental details 
    and information on how to reproduce the figures.
  \item \cref{apx:related-work} compares our results to standard convergence rates in the literature.
  \item \cref{apx:gd} gives the main results for gradient descent \cref{thm:gradient-descent}.
  \item \cref{apx:sign} gives the main results for sign descent \cref{thm:sign-rate}.
  \item \cref{apx:additional_details} gives the derivation for the time-to-$\varepsilon$ results of \cref{thm:informal-gd}.
\end{itemize}

\section{Experimental details}
\label{apx:experimental_details}

This section goes over the technical details of the experiments 
needed to reproduce the figures. 

\subsection{Computational complexity}
\label{apx:computational_complexity}
We use $d$ to denote the size of the vocabulary, but the number of parameters $\mW$ is $d^2$
as we have to learn the conditional probability table \smash{$\pi_{k\cond j}$}.
As the number of iterations $t$ has to scale with dimension, 
the problem scales in $d^3$, which becomes prohibitive fast.
To circumvent this issue, we use the fact that 
the training dynamics of gradient descent 
and sign descent on data following \cref{ass:sign-dynamics}
can be simulated in $O(d)$.
The error after $t$ iterations can then be computed in closed-form if initialized at~$0$,
making it possible to compute the loss after $t$ steps without computing the intermediate steps.

\begin{proposition}[Reduction of the dynamics for gradient descent]
\label{prop:simplified-dynamics-gd}
Under \cref{ass:conditional-distribution}, 
the dynamics of gradient descent with step-size~$1/\pi_1$ can be computed in $O(d)$ as 
\aligns{
  r_d(t) 
  \coloneqq
  \frac{\Loss_d(t) - \Loss_d^*}{\Loss_d(0)-\Loss_d^*}
  = 
  \frac{1}{\sum_{k=1}^d k^{-\alpha}} 
  \sum_{k=1}^d \frac{1}{k^\alpha} \paren{1-\frac{1}{k^\alpha}}^{2t}.
}
\end{proposition}
\begin{proof}
We use the the dynamics 
using the eigendecomposition notation
presented in \cref{sec:simplified_model},
\aligns{
  r_d(t) = \Loss_d(t) - \Loss_d^*
  = \sum_{i=1}^d \sum_{j=1}^d \lambda_{ij}\delta_{ij}(t)^2,
  &&\text{ and }&&
  \delta_{ij}(t) =
  = (1-\lambda_{ij})^t \delta_{ij}(0).
}
Using \cref{ass:conditional-distribution}
gives that $\lambda_{ij}$ is independent of $j$
and $\delta_{ij}$ is independent of $i$ as 
\aligns{
  \lambda_{ij} = \pi_i = \frac{1}{zi^{\alpha}}
  &&
  \delta_{ij}(0) = \pi_{\rho_i(j), i} = \frac{1}{zj^\alpha}
  &&
  \text{ where} &&
  z = \sum_{k=1}^d \frac{1}{k^\alpha}.
} 
Plugging in those together 
and using that the step-size is $\eta = \pi_1 = 1/z$
gives
\aligns{
  \frac{
    \Loss_d(t)-\Loss_d^* 
  }{
    \Loss_d(0)-\Loss_d^*
  }
  &= \frac{
    \sum_{i=1}^d \sum_{j=1}^d \frac{1}{z i^\alpha} \paren{1-\frac{1}{i^\alpha}}^{2t} \delta_{ij}(0)^2
  }{
    \sum_{i=1}^d \sum_{j=1}^d \frac{1}{z i^\alpha} \delta_{ij}(0)^2
  },
  \\
  &= \frac{
    \sum_{i=1}^d \frac{1}{i^\alpha} \paren{1-\frac{1}{i^\alpha}}^{2t} \sum_{j=1}^d  \delta_{ij}(0)^2
  }{
    \sum_{i=1}^d \frac{1}{i^\alpha} \sum_{j=1}^d \lambda_{ij}\delta_{ij}(0)^2
  }
  = \frac{
    \sum_{i=1}^d \frac{1}{i^\alpha} \paren{1-\frac{1}{i^\alpha}}^{2t} 
  }{
    \sum_{i=1}^d \frac{1}{i^\alpha} 
  }.
  \tag*{\qedhere}
}
\end{proof}

\begin{proposition}[Reduction of the dynamics for sign descent]
\label{prop:simplified-dynamics-sd}
Under \cref{ass:conditional-distribution}, 
the simplified dynamics of sign descent (\cref{ass:sign-dynamics})
with step-size $\eta(T, \phi) = 1/zT\phi^\alpha$ following the reparameterization of \cref{prop:step-size-range-for-sign}
where $z = \sum\!{}_{k=1}^d k^{-\alpha}$
can be computed in $O(d)$ as 
\aligns{
  r_d(T, \phi) \coloneqq
  \frac{\Loss_d(T, \eta(T, \phi)) - \Loss_d^*}{\Loss_d(0)-\Loss_d^*}
  = 
  \frac{1}{\sum_{k=1}^d k^{2\alpha}} 
  \sum_{k=1}^d 
  \paren{\switch{
    \frac{1}{k^\alpha} - \frac{1}{\phi^\alpha} & \text{ if } \abs{\delta_{ij}(T-1)} - \eta \geq 0,
    \\
    \frac{1}{2\phi^\alpha}  & \text{ otherwise},
  }
  }^2,
}
\end{proposition}
\begin{proof}
Using the same derivation as above for \cref{prop:simplified-dynamics-gd}
but using the update dynamics assumed in \cref{ass:sign-dynamics}.
Note that those dynamics imply that $\delta_{ij}(T)$ is independent of $i$.
Writing~$\Delta_j = \delta_{ij}(T)$ for any $i$ and using that $\sum_{i=1}^d \pi_i = 1$, we have
\aligns{
  \frac{
    \Loss_d(T)-\Loss_d^* 
  }{
    \Loss_d(0)-\Loss_d^*
  }
  &= \frac{
    \sum_{i=1}^d \sum_{j=1}^d \lambda_{ij}\delta_{ij}(T)^2
  }{
    \sum_{i=1}^d \sum_{j=1}^d \lambda_{ij}\delta_{ij}(0)^2
  }
  = \frac{
    \sum_{i=1}^d \sum_{j=1}^d \pi_i \Delta_{j}(T)^2
  }{
    \sum_{i=1}^d \sum_{j=1}^d \pi_{i} \Delta_{j}(0)^2
  }
  = \frac{
    \sum_{j=1}^d \Delta_{j}(T)^2
  }{
    \sum_{j=1}^d \Delta_{j}(0)^2
  }.
}
Expanding $\Delta_j(T)$ using \cref{ass:sign-dynamics} gives the result.
\end{proof}

For the real data experiments in \cref{fig:actual-vs-predicted-rates},
the dynamics cannot be reduced to a $O(d)$. 
We still use the fact that the dynamics can be computed in closed-form 
to avoid running $t$ steps of gradient/sign descent.
For sign descent, we use the following equation
for the loss after $t$ steps of sign descent 
(not the simpler model of \cref{ass:conditional-distribution})
by computing the point at which the loss oscillates.
\begin{proposition}
Under the dynamics of sign descent with step-size $\eta$, 
\aligns{
  \delta_{ij}(t+1) = \delta_{ij}(t) - \eta \sign(\delta_{ij}(t)),
}
the distance after $t$ steps is given by 
\aligns{
  \delta_{ij}(t) = 
  \switch{
    \delta_{ij}(0) - \eta t & \text{ if } t \leq T_{\mathrm{switch}},\\
    c_{ij} & \text{ if } t - T_{\mathrm{switch}} \text{ is odd},\\
    c_{ij}-\eta & \text{ if } t - T_{\mathrm{switch}} \text{ is even},
  }
  &&\text{ where }&&
  \begin{aligned}
  T_{\mathrm{switch}} &= \floor{\delta_{ij}(0) / \eta},\\
  c_{ij} &= \delta_{ij}(0) - T_{\mathrm{switch}}\eta.
  \end{aligned}
}
\end{proposition}

\subsection{Additional details about the figures}

\textbf{\cref{fig:motivation}}
shows the dynamics of gradient descent on \cref{prob:bigram} 
on data satisfying \cref{ass:conditional-distribution}.

\textbf{\cref{fig:actual-vs-predicted-rates}}
shows the dynamics on real data on the OpenWebText dataset~\citep{Gokaslan2019OpenWeb}.
Using the SentencePiece~\citep{kudo2018sentencepiece} implementation of BPE~\cite{senrich2016bpe}, 
we train tokenizers with vocabulary sizes of 
$1\,000$, $3\,612$, $10\,000$ and $31\,622$ tokens
on a the first $2\,000\,000$ entries of the dataset with a maximum sentence length of $16\,768$.
We compute the frequencies and conditional frequency tables 
for each vocabulary size using the entire dataset.
We use the closed form formulas for the loss after $t$ steps 
using $O(d^2)$ computation detailed in the previous section 
to avoid having to run gradient and sign descent on those large models.

Gradient descent uses the empirically-derived step-size of $1/\pi_1$. 
For sign descent, for a given time horizon $T$,
we optimize over the step-size numerically. 
Because the loss after $T$ steps as a function of the step-size is unimodal,
we use the default bounded bracketing method in scipy~\citep[][\href{https://docs.scipy.org/doc/scipy/reference/generated/scipy.optimize.minimize_scalar.html}{\texttt{minimize\_scalar}}]{2020SciPy-NMeth}
starting with the interval $[\eta_{\min}/d, d\eta_{\max}]$
where $\eta_{\min}, \eta_{\max}$ are the bounds derived in \cref{prop:step-size-range-for-sign}.
The optimal step-size can vary drastically if it is computed on even or odd iterations
as the loss oscillates.
To avoid this issue, we only show even iterations.

\textbf{\cref{fig:assumption-is-reasonable}}
shows the frequencies computed as for \cref{fig:actual-vs-predicted-rates}
for the largest vocabulary size, $d = 31\,622$.

\textbf{The rightmost plot of \cref{fig:sign-assumption}} shows the simplified dynamics of sign descent.

\textbf{\cref{fig:gd-rates}, \cref{fig:sign-stepsize-convergence} and \cref{fig:sign-loss-convergence}}
show the convergence of the loss in $d$ dimension computed using 
the equations in \cref{apx:computational_complexity}.
For sign descent, the best step-size 
is obtained by grid search.
We know the optimal step-size satisfies~$\phi \in [1, d]$ (\cref{prop:step-size-range-for-sign}),
so let~$\phi = d^x$ where $x$ 
comes from a logarithmically spaced grid-search on $x$ from $-10$ to $0$, 
taking every $1/32$th powers;
\aligns{
  \phi \in \{d^x : x \in \{10^{-10}, 10^{-10+\frac{1}{32}}, 10^{-10+\frac{2}{32}}, \ldots, 10^0\}\}.
}

\clearpage
\section{Comparison with worst-case rates}
\label{apx:related-work}

In this section, we compare our rates against results obtained using classical analyses 
to highlight the benefit of the asymptotic analysis in capturing the dependence on dimension.
Our goal is not to imply those bounds are poor;
each of the work cited below studied a specific problem
and the assumptions were selected to highlight the impact of the condition number, non-convexity, variance, or other issue.
However, due to their worst-case generality, 
existing results do not capture the dimension dependence on the problem of the linear bigram problem (\cref{prob:bigram})
with Zipf-distributed frequencies (\cref{ass:conditional-distribution})
and predict worse behavior than actually observed.

In this section, we focus on Zipf-distributed data ($\alpha = 1$)
as it is the most relevant to text data. %
To simplify notation, 
we assume that the conditional frequencies directly 
follow a power-law $\pi_{k\cond i} \propto 1/k$, 
instead of assuming that there exists a reordering $\rho_i$ such that \smash{$\pi_{\rho_i(k)\cond i} \propto 1/k$}
as in \cref{ass:conditional-distribution}.
This reordering does not affect the dynamics of the loss and can be ignored without loss of generality.

\subsection{Standard smooth, (strongly-)convex rates.}
Classical results in smooth, convex optimization are derived
under the assumption that the objective function $\Loss_d$ is $L$-smooth and $\mu$-strongly convex with $\mu \geq 0$.
We write the function rates in matrix form for the loss $\Loss_d$ defined in \cref{prob:bigram}, 
but this could equivalently be transformed to a vector form using and $\norm*{\vx - \vx_*}^2_2 = \norm*{\mW-\mW_*}^2$ if
$\vx = \mathrm{vec}(\mW)$ and $\vx_* = \mathrm{vec}(\mW_*)$ where $\mathrm{vec}$ stacks the columns of $\mW$ as a single vector.
For a twice-differentiable function, 
this is equivalent to assuming that the eigenvalues of the Hessian 
are bounded by $\mu \leq \lambda_{ij} \leq L$ for all $i,j\in [d]$ at every possible input.
We compare against simple forms available in this setting
(\citet[][Cor.~2.1.2]{nesterov2018intro}, \citet[][Eq.~9.18]{Boyd_Vandenberghe_2004}).
While is possible to slightly improve the constants in these bounds, 
these constants do not meaningfully affect the asymptotic behavior as $d$ grows.
\aligns{
  \optgap \leq \frac{2L\norm{\mW_0 - \mW_*}^2_F}{t},
  &&
  \optgap \leq \paren{1-\frac{\mu}{L}}^{t} \paren{\optgapStart}.
}
To better compare these rates with our results, we normalize them by $\optgapStart$,
\aligns{
  \frac{\optgap}{\optgapStart}
  \leq 
  \frac{L\norm{\mW_0 - \mW_*}^2_F}{t\paren{\optgapStart}}
  \eqcolon
  r^{\mathrm{sub}}_d(t),
  &&
  \frac{\optgap}{\optgapStart} \leq \paren{1-\frac{\mu}{L}}^{t}
  \eqcolon
  r^{\mathrm{lin}}_d(t).
}
\begin{proposition}[Values of the constants]
On \cref{prob:bigram} with frequencies following a power-law with $\alpha = 1$ (\cref{ass:conditional-distribution})
initialized at $\mW_0 = 0$, 
the smooth convex sublinear rate $r^{\mathrm{sub}}_d(t)$
and the smooth strongly-convex linear rate $r^{\mathrm{lin}}_d(t)$
are asymptotically equivalent to 
\aligns{
  r^{\mathrm{sub}}_d(t)
  \asym{d}
  2
  \frac{d}{\log(d)}
  \frac{1}{t},
  &&
  r^{\mathrm{lin}}_d(t)
  \asym{d}
  \paren{1-\frac{1}{d}}^{t}.
}
\label{prop:constants-for-gd}
\end{proposition}
\begin{proof}
The proof follow from substituting the constants with the values
\aligns{
  \mu = \frac{1}{dz},
  &&
  L = \frac{1}{z},
  &&
  \norm{\mW_0 - \mW^*}^2_F
  =
  d \paren{\Loss_d(\mW_0) - \Loss_d^*}.
}
where $z = \sum_{k=1}^d 1/k \asym{d} \log(d)$.
The eigenvalues are $\lambda_{ij} = \pi_i = \nicefrac{1}{zi}$ after normalization, 
giving $L=\nicefrac{1}{z}$ and $\mu=\nicefrac{1}{zd}$.
Using that $\delta_{ij}(0) = 1/zj$ gives
the loss and distance at initialization,
\aligns{
  {\Loss_d(\mW_0) - \Loss_d^*} 
  &= \sum_{i=1}^d\sum_{j=1}^d \lambda_{ij} \delta_{ij}(0)^2 
  = \sum_{i=1}^d \pi_i \sum_{j=1}^d \pi_{j\cond i}^2 
  = \sum_{j=1}^d \paren{\frac{1}{zj}}^2,
  \\
  \norm{\mW_0-\mW_*}^2 
  &= \sum_{i=1}^d \sum_{j=1}^d \delta_{ij}(0)^2
  = d \sum_{j=1}^d \paren{\frac{1}{zj}}^2
  = d \paren{\Loss_d(\mW_0) - \Loss_d^*}.
  \tag*{\qedhere}
}
\end{proof}

Both rates struggle to predict the progress in ``early'' iterations, when $t$ is much smaller than $d$.
The sublinear rate requires a scaling $t \propto \nicefrac{d}{\log(d)}$ while the linear rate predicts $t \propto d$.
Neither captures the progress that can be made by running $t = \smash{d^{1/2}}$ iterations, 
which reaches an error of $\varepsilon=1/2$.
Instead, both rates predict no progress.
We visualize the given rates in \cref{fig:compare-rate}
after rescaling the number of steps 
to our normalized time $\tau = \log(t)/\log(d)$.
The linear and sublinear rates are not converging to $r(\tau) = 1-\tau$.
Instead, they exhibit a sharper and sharper transition 
between not predicting any progress for $\tau < 1$
($r(\tau) \approx 1$ or $r(\tau) > 1$)
and that the problem is solved if $\tau > 1$.

\begin{figure}
\includegraphics{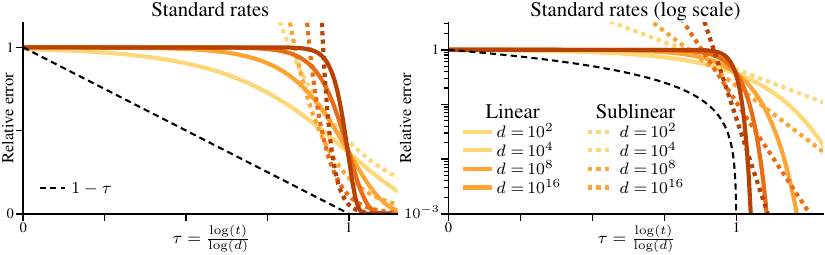}
\caption{\textbf{Standard convergence rates don't capture the scaling in dimension.}}
\label{fig:compare-rate}
\end{figure}

\subsection{Rates for sign descent}

Analyses on sign-like methods in the literature 
typically target more complex algorithms such as RMSProp~\citep{tieleman2012rmsprop} or AdaGrad~\citep{duchi11adaptive} for \citet{das2024towards,liu2025adagradanisotropic},
or consider more general problems including non-convex functions for \citet{bernstein2018signsgd,safaryan2021ssd}.
We are not aware of existing analyses that specifically target sign descent on diagonal quadratic problems such as \cref{prob:bigram}.
This makes a direct comparison difficult.
It might be that the rates described in those papers for the chosen problem setting or algorithm are tight.
However, our message is that the resulting rates are too pessimistic 
even for a problem as simple as \cref{prob:bigram} and suggest runtimes for sign descent 
that are off by a factor depending on the dimension.

The main difficulty in studying sign descent and sign-like methods more generally 
is the strong dependence on the coordinate system used.
For \cref{prob:bigram} the dynamics perfectly separate along coordinates 
which makes it possible to derive a closed form for the dynamics.
Other works typically rely on assumptions on the Hessian 
that quantify how close to diagonal it is. 
For example, bound the Hessian with a diagonal matrix $\mL$, 
$\mH \preceq \mL$ in Loewner ordering, 
and obtain rates that depend on the trace of $\mL$
\citep[e.g.,][]{bernstein2018signsgd,liu2025adagradanisotropic}. 
For \cref{prob:bigram}, the Hessian is diagonal and made of $d$ diagonal copies
of $\mX^\mtop \mX/n = \Diag([\pi_1, ..., \pi_d])$, thus $\Tr(\mL) = \Tr(\nabla^2 \Loss_d(\mW)) = d$.

\textbf{Anisotropic smoothness and AdaGrad.}
Using this assumption, \citet[][Theorem 4.1]{liu2020transformers} show the following convergence rate for AdaGrad.
To simplify their results and show the rate in its best light,
we assume there is no noise in the gradient~($\norm{\vsigma}_1 = 0$ in their notation),
that AdaGrad is run with the parameter $\epsilon = 0$, 
that the algorithm is run with projections onto the constrained set~$\setW = \{\mW : \norm{\mW}_\infty \leq \pi_1 \}$
and that we initialize at $\mW = 0$.
\aligns{
  \Loss_d(t) - \Loss_d^* \leq \frac{\Tr(\mL) \pi_1}{T}.
}
Normalizing the loss 
and simplifying the constants 
using the same approach as in \cref{prop:constants-for-gd}
gives the following asymptotic upper bound
\aligns{
  \frac{\Loss_d(t) - \Loss_d^*}{\Loss_d(0)-\Loss_d^*} \leq 
  r^{\mathrm{Adagrad}}_d(t) \coloneq
  \frac{\Tr(\mL)\pi_1}{T \paren{\Loss_d(0)-\Loss_d^*}}
  \asym{d}
  \frac{d \log(d)}{T}\frac{6}{\pi^2}.
}
Although we might expect Adagrad to outperform sign descent as it uses decreasing step-sizes to avoid the oscillations,
this rate estimate that the number of iterations should scale with $d\log(d)$
instead of the scaling of $\sqrt{d}$ we find for sign descent. 

\textbf{Preconditioning effect of Adam.}
\citet{das2024towards}
study RMSProp, or Adam without momentum~($\beta_1=0$) but with momentum on the moving average of the squared gradient.
They use high-probability arguments to handle the dynamics of the preconditioner and random initialization.
Their rate shows that Adam can perform better on diagonal quadratics 
if the condition number scales worse than linearly with the dimensionality, 
by replacing the condition number $\kappa$ 
with  $\kappa_{\mathrm{Adam}} = \min\{d_{\mW}+1,\kappa\}$
where $d_{\mW}$ is the dimensionality of $\mW$.
Assuming that their bound holds with probability $1$ with $\mW_0=0$ and ignoring logarithmic factors in $d$ and $\epsilon$, 
their rate for diagonal quadratics is~\citep[][Thm. 2]{das2024towards} 
\aligns{
  \Loss_d(t) - \Loss_d^* \leq \frac{\epsilon^2}{2},
  &&\text{ after }&&
  t \geq \tilde{O}\paren{\kappa_{\mathrm{Adam}}}.
}
Unfortunately, on \cref{prob:bigram} the dimensionality is $d_{\mW} = d^2$
while the condition number scales as $\kappa = d$ with Zipfian eigenvalues ($\alpha = 1$)
so the proposed approach does not improve over gradient descent.
Normalizing the loss and using the same approach as in \cref{prop:constants-for-gd}
gives 
\aligns{
  \frac{\Loss_d(t) - \Loss_d^*}{\Loss_d(0) - \Loss_d^*} \leq \frac{\epsilon^2}{2},
  &&\text{ after }&&
  t \geq \tilde{O}\paren{d}.
}
This scaling predicts the same performance for Adam and gradient descent 
(up to log factors depending on $d$ and $\epsilon$ that we ignored) 
whereas our analysis shows a scaling of ${d}^{1/2}$ for sign descent.

\textbf{Non-convex results.}
Results in the non-convex setting~\citep{bernstein2018signsgd,balles2020geometrysigndescent,safaryan2021ssd,liu2025adagradanisotropic}
give convergence results to stationarity instead of convergence in optimality gap,
measured using the 1-norm of the gradient instead of the Euclidean norm.
Because $\norm*{\vv}_1^2 \leq \norm*{\vv}_2^2 d$ for a $d$-dimensional vector $\vv$, 
the time required to get the $1$-norm small might be much worse 
than the time required to find a stationary point in Euclidean norm or to minimize the function value.
To illustrate this point, we show that it is possible to have arbitrarily small relative error on \cref{prob:bigram} 
and arbitrarily large gradients when measured in 1-norm in high enough dimension.

\begin{proposition}
On \cref{prob:bigram} with Zipf-distributed data~(\cref{ass:conditional-distribution} with~$\alpha = 1$),
sign descent with simplified dynamics~(\cref{ass:sign-dynamics})
using the scaling $t_d(\tau) = \tau d^{1/2}/2$
and $\phi_d(\tau) = (1+\tau^2)^{-1}$
satisfies for $\tau > 2$
\aligns{
  \frac{\Loss_d(\mW_{t_d(\tau)}) - \Loss_d^*}{\Loss_d(\mW_0) - \Loss_d^*}
  \asym{d}
  \frac{1}{1+\zeta(2\alpha) \tau^2},
  &&
  \frac{\norm{\mathrm{vec}(\nabla\Loss_d(\mW_{t_d(\tau)}))}_1}{\norm{\mathrm{vec}(\nabla\Loss_d(\mW_0))}_1}
  \asym{d} 
  C\frac{d^{1/2}}{\log(d)\tau }
  \, \text{ where } \, 1/2< C<1.
}
\end{proposition}
\begin{proof}
Computations similar to \cref{prop:loss-of-sd}
show that the 1-norm of the gradient is
\aligns{
  \norm{\mathrm{vec}(\nabla\Loss_d(\mW_t))}_1 = \sum_{k=1}^{k_*} \paren{\pi_k - t\eta} + \sum_{k=k_*+1}^d \frac{\eta}{2}
}
where $k_*$ is the number of directions that are still in the decreasing regime after $T$ steps with step-size~$\eta$.
As~\smash{$\norm{\mathrm{vec}(\nabla\Loss_d(\mW_0))}_1 = \sum\!{}_{k=1}^d \pi_k = 1$}, 
this expression is also the normalized 1-norm of the gradient.
Using the parameterization $\eta = 1/zt\phi$, 
where $z = \sum{}\!_{k=1}^d 1/k$, we get the update
\aligns{
  r_d(t) \coloneqq
  \frac{\norm{\mathrm{vec}(\nabla\Loss_d(\mW_t))}_1}{\norm{\mathrm{vec}(\nabla\Loss_d(\mW_0))}_1}
  = 
  \frac{
    \sum_{k=1}^{\floor{\phi}} \paren{\frac{1}{k} - \frac{1}{\phi}} + \sum_{k=\floor{\phi}+1}^d \frac{1}{2t\phi}
  }{z}
}
Using $t_d(\tau) = \frac{1}{2}\tau d^{1/2}$ and $1 < \phi_d(\tau) < 2$ for simplicity
and that $z \sim^d \log(d)$ gives
\aligns{
  r_d(t) \asym{d} \frac{d^{1/2}}{\log(d) \tau} C
  \text{ where } 
  \frac{1}{2}<C<1.
}
\end{proof}
Getting the magnitude of the gradient in 1-norm smaller than a constant independent of $d$ would require scaling $t$ with $d/\log(d)$, 
whereas getting the same result for the relative error only requires scaling $t$ with $d^{1/2}$.

\clearpage
\subsection{Source and capacity assumptions}
The classical source/capacity condition have typically been used to describe 
risk bounds in learning theory
for infinite dimensional kernel methods, random feature models 
or regression models when the dimension $d$ grows~\citep[e.g.,][]{caponnetto2007source,advani2020highdim,berthier2020tight,bahri2021explaining,cui2021generalization,maloney2022solvable,paquette2024phases}.
Using the notation of \citet{cui2021generalization}
where $\mSigma$ is the covariance of the data 
and $\vtheta_*$ is the solution, 
the typical assumption is that for parameters $\alpha > 1, r \geq 0$ we have
\aligns{
  \Tr(\mSigma^{\frac{1}{\alpha}}) < \infty,
  && 
  \vtheta_*^\mtop \mSigma^{1-2r}\vtheta_* < \infty.
}
For finite dimensionals system, $\Tr(\mSigma^{1}{\alpha}) < \infty$
should be interpreted as $\lim_{d\to\infty} \sum{}\!_{k=1}^d \lambda_k^{1/\alpha} < \infty$,
where $\lambda_k$ are the eigenvalues of $\mSigma$.
This implies that the data is ``effectively'' low-dimensional, even as $d \to \infty$.
\citet{cui2021generalization} further assume that the eigenvalues of $\mSigma$ and the components of $\vtheta_*$ follow power laws,
\aligns{
  \lambda_k = k^{-\alpha},
  &&
  [\vtheta_*]_k^2 = k^{-1-\alpha(2r-1)}.
}
Our setting follows a similar idea.
The parameter $\alpha$ corresponds to our power law parameter $\alpha$ governing the conditional frequencies
and the distance to the solution corresponds to the marginal probabilities. 
In terms of assumptions, the main difference is 
that we normalize the eigenvalues and the distance to the solutions 
as those represent frequencies, and that we allow for $\alpha \leq 1$
to describe Zipf-distributed data.
As we have $d^2$ eigenvalues, where each distinct value is repeated $c$ times, 
we could collapse to a $d$-dimensional system with distances to the solution summed to obtain
\aligns{
  \lambda_k = \pi_k = \frac{1}{z}k^{-\alpha},
  &&
  [\vtheta_*]_k^2 = \sum_{j=1}^d \pi_{j\cond k}^2 = \frac{1}{z^2} \sum_{j=1}^d j^{-2\alpha}.
}
Up to the normalization constant $z$,
we recover the source/capacity condition with the same $\alpha$ and $r = 1/2$
as the distance to the solution $[\vtheta_*]_k^2$ is independent of $k$.
The main conceptual difference is that we use this setting to investigate 
the performance of deterministic optimization algorithms, as~\citet{velikanov2024tight},
instead of investigating risk bounds.

\clearpage
\section{Proofs for gradient descent}
\label{apx:gd}

This section gives the proof of \cref{thm:gradient-descent} for the scaling of gradient descent.

\subsection{Standard results}

We start with standard results that are used in the subsequent proofs.
The following classical relationships between sums and integrals of monotone functions 
will be used to bound the approximation error 
induced by analyzing the asymptotics of the integral 
instead of the sum.

\begin{lemma}[Sum-Integral]
\label{lem:sum-and-integral}
For a function $f$ that is monotone on $[a,b]$,
\aligns{
  \text{if $f$ is increasing on $[a,b]$,}
  &&
  \sum_{i=a}^{b-1} f(k) 
  \leq
  \int_a^b f(k)  \,\dif{k}
  \leq 
  \sum_{i=a+1}^{b} f(k),
  \\
  \text{if $f$ is decreasing on $[a,b]$,}
  &&
  \sum_{i=a+1}^{b} f(k) 
  \leq
  \int_a^b f(k)  \,\dif{k}
  \leq 
  \sum_{i=a}^{b-1} f(k).
}
\end{lemma}
\newcommand{\LemmaSumIntegral}{\hyperref[lem:sum-and-integral]{Sum-Integral Lemma} (\ref{lem:sum-and-integral})}

To apply these sum-integral relationships to the dynamics of gradient descent in \cref{thm:gradient-descent},
we need to describe when they are increasing or decreasing.

\begin{lemma}[Unimodal sequence]%
\label{lem:increasing-decreasing}
The sequence $s(k) = k^{-\alpha}(1-k^{-\alpha})^t$ is non-negative on $k\geq 1$ and unimodal. 
It monotonically increases until $k_* = \paren*{1 + t}^{1/\alpha}$, 
then monotonically decreases.
\end{lemma}
\newcommand{\LemmaUnimodal}{\hyperref[lem:increasing-decreasing]{Unimodal Lemma} (\ref{lem:increasing-decreasing})}
\begin{proof}
As $s(k)$ is non-negative, we can instead look at its logarithm,
\aligns{
  \log s(k)
  &=
  \log(N) -\alpha \log(k) + t \log(1-k^{-\alpha}) ,
  \\
  \dd{k}\log s(k) 
  &= 
  \alpha t \frac{k^{-\alpha-1}}{1-k^{-\alpha}} 
  -\frac{\alpha}{k}
  = 
  \frac{\alpha t}{k(k^{\alpha}-1)} 
  -\frac{\alpha}{k}
  = 
  \frac{\alpha (t-1)(k^\alpha-1)}{k(k^{\alpha}-1)}.
}
The denominator is positive on $k \geq 1$, 
and the numerator is positive for small $k$ until 
the derivative changes sign at $\alpha t - \alpha(k^{\alpha}-1) = 0$, or
$k_* = \paren{1+t}^{1/\alpha}$.
\end{proof}

At she partial sum $H_{d,\alpha} = \sum\!{}_{k=1}^d k^{-\alpha}$,
appears in the proof of gradient and sign descent,
we give its asymptotic behavior independently.

\begin{lemma}[Normalizer Asymptotics]
\label{lem:normalizer}
As $d$ grows, the partial sum $H_{d,\alpha} = \sum_{k=1}^d k^{-\alpha}$ behaves as 
\aligns{
  H_{d,\alpha} \asym{d} \tfrac{1}{1-\alpha}d^{1-\alpha}
  \, \text{ if } \,
  \alpha < 1,
  &&
  H_{d,1} \asym{d} \log(d)
  &&
  H_{d,\alpha} \asym{d} \zeta(\alpha)
  \, \text{ if } \,
  \alpha > 1,
}
where $\zeta$ is the zeta function,
defined as the limit of $H_{d,\alpha}$, 
$\zeta(\alpha)=\sum_{k=1}^\infty k^{-\alpha} < \infty$ for $\alpha > 1$.
\end{lemma}
\newcommand{\LemmaNormalizer}{\hyperref[lem:normalizer]{Normalizer Asymptotics Lemma} (\ref{lem:normalizer})}
\begin{proof}
For $\alpha > 1$, the sum converges to $\sum_{k=1}^\infty k^{-\alpha} = \zeta(\alpha)$.
For $\alpha \leq 1$, the sum diverges as $d$ grows.
As the sequence $k^{-\alpha}$ is decreasing in $k$,
we can use the~\LemmaSumIntegral{} to get
\aligns{
  \int_1^{d+1} k^{-\alpha} \dif{k}
  \leq 
  \sum_{k=1}^d k^{-\alpha} 
  \leq 1 + \int_1^d k^{-\alpha} \dif{k}.
}
If $\alpha < 1$, the integrals evaluate to
\aligns{
  \frac{\paren{(d+1)^{1-\alpha}-1}}{1-\alpha}
  \leq 
  \sum_{k=1}^d k^{-\alpha} 
  \leq
  \frac{\paren{d^{1-\alpha}-1}+1}{1-\alpha},
}
and both terms are asymptotically equivalent to $d^{1-\alpha}/(1-\alpha)$ as $d \to \infty$.
If $\alpha = 1$, this gives
\aligns{
  \log(d+1)
  \leq 
  \sum_{k=1}^d k^{-\alpha} 
  \leq
  \log(d)+1.
}
Both terms are asymptotically equivalent to $\log(d)$.
\qedhere
\end{proof}

\clearpage

The main purpose of the \LemmaSumIntegral{} and the \LemmaUnimodal{}
is to bound on the error incurred 
by approximating the sum with the integral form of the loss.

\begin{lemma}[Approximating error]
\label{lem:approx-error}
{}The approximation error between the following sum and integral, 
\aligns{
  S_d(t) = 
  \sum_{k=1}^d s(k)
  &&
  I_d(t) = 
  \int_{1}^d s(k) \dif{k}
  && \text{ where } && 
  s(k) = k^{-\alpha}(1-k^{-\alpha})^t 
}
can be bounded by the following error term,
\alignn{
  \label{eq:approx-error}
  \abs{S_d(t) - I_d(t)}
  \leq \delta_d(t) 
  \quad \text{ where } \quad 
  \delta_d(t) 
  \coloneq
  \switch{
    \frac{1}{1+t}
    \paren{1-\frac{1}{1+t}}^t
    & \text{ if } 1+t \leq d^\alpha,
    \\
    \frac{1}{d^\alpha}\paren{
      1-\frac{1}{d^\alpha} 
    }^t
    & \text{ if } 1+t \geq d^\alpha.
  }
}
\end{lemma}
\newcommand{\LemmaApproxError}{\hyperref[lem:approx-error]{Approximation Error Lemma} (\ref{lem:approx-error})}
\begin{proof}
By the \LemmaUnimodal{}, the sequence $s(k)$ is increasing until $k_* = \paren*{1 + t}^{1/\alpha}$ then decreasing,
which lets us use the \LemmaSumIntegral{}.

\textbf{For large $t$.}
Suppose that $t$ is sufficiently large such that $k_* \geq d$ and $1+t \geq d^\alpha$,
meaning that the sequence $s(k)$ is increasing on $[1, d]$. Then,
\alignn{
  \int_{1}^d s(k) \,\dif{k} + s(1)
  \leq 
  \sum_{1=1}^d s(k)
  \leq 
  \int_{1}^d s(k) \,\dif{k} + s(d).
}
Using that $s(1) = 0$ gives $\abs{I_d(t)-S_d(t)} \leq s(d)$ when $t$ is large.

\textbf{For small $t$.}
If $t$ is small and $k_* < d$
the sequence flips from increasing to decreasing on $[1, d]$.
We still use the same idea, but bound the increasing and the decreasing subsequences separately.

\textbf{Upper bound.}
As the sequences $s(k)$ in increasing on $[1,k_*]$ and decreasing on $[k_*, d]$,
\aligns{
  \sum_{k=1}^{\floor{k_*}-1} s(k)
  \leq 
  \int_{1}^{\floor{k_*}}s(k) \,\dif{k},
  &&
  \sum_{k=\floor{k_*}+2}^{d} s(k)
  \leq 
  \int_{\floor{k_*}+1}^d s(k) \,\dif{k}.
}
Summing both bounds and adding the remaining terms 
$s(\floor{k_*}), s(\floor{k_*}+1)$, 
\aligns{
  \sum_{k=1}^{d} s(k)
  &\leq 
  \int_{1}^{\floor{k_*}}
  \!\!\!\! s(k) \,\dif{k}
  +
  \int_{\floor{k_*}+1}^d 
  \!\!\!\! s(k) \,\dif{k}
  + s(\floor{k_*})
  + s(\floor{k_*}+1)
  \leq 
  \int_{1}^{d} 
  \!\! s(k) \,\dif{k} + s(k_*),
}
where the last inequality uses the following simplifications,
\aligns{
  \min\{s(\floor{k_*}), s(\floor{k_*}+1)\} 
  &=
  \int_{\floor{k_*}}^{\floor{k_*}+1} \min\{s(\floor{k_*}), s(\floor{k_*}+1)\} \dif{k} 
  \leq 
  \int_{\floor{k_*}}^{\floor{k_*}+1} s(k) \dif{k},
  \\
  \max\{s(\floor{k_*}), s(\floor{k_*}+1)\} &\leq s(k_*).
}

\textbf{Lower bound.}
Now using the lower bound,
\aligns{
  \int_{1}^{\floor{k_*}}s(k) \,\dif{k}
  \leq 
  \sum_{k=2}^{\floor{k_*}} s(k),
  &&
  \int_{\floor{k_*}+1}^d s(k) \,\dif{k}
  \leq 
  \sum_{k=\floor{k_*}+1}^{d-1} s(k).
}
Summing both bounds, we can complete the integral 
by adding and subtracting \smash{$\int_{\floor{k_*}}^{\floor{k_*}+1} s(k)\,\dif{k}$}
and adding the remaining terms $s(1)$ and $s(d)$ to obtain
\aligns{
  \sum_{k=1}^{\floor{k_*}} s(k)
  &\geq
  \int_{1}^{d} s(k) \,\dif{k}
  - \int_{\floor{k_*}}^{\floor{k_*}+1} s(k)\,\dif{k}
  + s(1) + s(d)
  \geq 
  \int_{1}^{d} s(k) \,\dif{k} -s(k_*) + s(d),
}
where the last inequality uses that $s(1)=0$, $s(k) \leq s(k_*)$.

\textbf{Combining the results} for the small $t$ regime gives
\aligns{
  I_d(t) + s(k_*) \geq S_d(t) \geq I_d(t) - s(k_*) + s(d),
  \quad \text{ so } \quad 
  \abs{I_d(t)-S_d(t)} \leq s(k_*).
}
\textbf{The final bound} in \cref{eq:approx-error} expands $s(x) = x^{-\alpha}(1-x^{-\alpha})^t$
and replaces $k_*$ by $(1+t)^{\frac{1}{\alpha}}$.
\qedhere
\end{proof}

\clearpage
\subsection{Scaling laws for gradient descent}
\label{apx:gd-main}

We are now ready to move to the proof of \cref{thm:gradient-descent},
for which we recall the theorem statement.
\thmgd*

\begin{proof}[Proof sketch]
We first give a sketch of the proof, 
which will be formalized in the next lemmas.
Based on the reduced dynamics 
for gradient descent in \cref{prop:simplified-dynamics-gd}, we know that 
\aligns{
  r_d(t) = \frac{\Loss_d(t)-\Loss_d^*}{\Loss_d(0)-\Loss_d^*}
  = \frac{\sum_{k=1}^d k^{-\alpha}(1-k^{-\alpha})^{t}}{H_{d,\alpha}},
}
where $H_{d,\alpha} = \sum_{k=1}^d k^{-\alpha}$.
Let $S_d$ and $I_d$ be the sum and integral variants of the denominator,
\alignn{
  \label{eq:integral-form}
  S_d(t) = \sum_{k=1}^d k^{-\alpha}(1-k^{-\alpha})^{t}
  &&
  I_d(t) = \int_1^d k^{-\alpha}(1-k^{-\alpha})^{t} \dif{k}.
}
First, we establish in \cref{lem:asymptotic-integral}
that the integral form converges to the rate $r(\tau)$ in \cref{thm:gradient-descent},
\aligns{
  \lim_{d\to\infty} \frac{I_d(t_d(\tau))}{H_{d,\alpha}} = r(\tau).
}
Next, we show in \cref{lem:approximation-error-negligible}
that the error incurred by approximating the sum $S_d$ 
by the integral $I_d$ is negligible,
in the sense that 
$\abs{I_d(t)-S_d(t)}\leq\delta_d(t)$ and 
\aligns{
  \lim_{d\to\infty} 
  \frac{
    \delta_d(t_d(\tau))
  }{
    I_d(t_d(\tau))
  }
  = 0 
  \quad 
  \text{ if } \alpha \leq 1,
  &&\text{ and }&&
  \lim_{\tau\to\infty} 
  \lim_{d\to\infty} 
  \frac{
    \delta_d(t)
  }{
    I_d(t)
  }
  = 0 
  \quad 
  \text{ if } \alpha > 1.
}
This gives the results 
that 
\aligns{
  r(\tau) 
  = \lim_{d\to\infty} \frac{I_d(t_d(\tau))}{H_{d,\alpha}}
  \text{ if } \alpha \leq 1,
  && \text{ and } &&
  r(t) 
  \asym{\tau} \lim_{d\to\infty} \frac{I_d(t)}{H_{d,\alpha}}
  \text{ if } \alpha > 1.
}
with the values of $r(\tau)$ given in \cref{thm:gradient-descent}.
\end{proof}

\clearpage
\begin{lemma}[Asymptotics of the integrals]
\label{lem:asymptotic-integral}
Let $I_d(t)$ be the integral form given in \cref{eq:integral-form}
and $t_d(\tau)$ be the scaling given in \cref{thm:gradient-descent}.
The following limits hold.
\aligns{
  \text{If } \alpha < 1,
  &&
  t_d(\tau) &= \tfrac{1}{2}\tau d^\alpha,
  &
  \lim_{d\to\infty}\frac{I_d(t_d(\tau))}{H_{d,\alpha}}
  &=
  \frac{1-\alpha}{\alpha}E_{\frac{1}{\alpha}}(\tau)
  \asym{\tau}
  \frac{1-\alpha}{\alpha} \frac{e^{-\tau}}{\tau+1},
  \\
  \text{if } \alpha = 1,
  &&
  t_d(\tau) &= \tfrac{1}{2}d^\tau,
  &
  \lim_{d\to\infty}\frac{I_d(t_d(\tau))}{H_{d,\alpha}}
  &=
  1 - \tau
  \quad \quad 
  \text{ where } \tau \in [0,1],
  \\
  \text{if } \alpha > 1,
  &&
  t_d(\tau) &= \tau,
  &
  \lim_{d\to\infty}\frac{I_d(t_d(\tau))}{H_{d,\alpha}}
  &=
  \frac{B\paren{1-\frac{1}{\alpha}, 1+2t}}{\alpha \zeta(\alpha)}
  \asym{\tau}
  C
  \frac{1}{\tau^{1-\frac{1}{\alpha}}}
  \Loss_d(0),
}
\end{lemma}
\begin{proof}
\textbf{For $\alpha > 1$.} We use the change of variable $z = k^{-\alpha}$ to get 
\aligns{
  I_d(t) = \frac{1}{\alpha} \int_{d^{-\alpha}}^1 z^{-\frac{1}{\alpha}} (1-z)^t \dif{z}
}
As $d \to \infty$, the integral converges to definition of the Beta function 
\aligns{
  \lim_{d\to\infty} \alpha I_d(t) = \int_0^1 z^{-\frac{1}{\alpha}} (1-z)^t \dif{z} 
  \eqcolon B\paren{1-\frac{1}{\alpha}, 1+t}.
}
As $\lim_{d\to\alpha} H_{d,\alpha} = \zeta(\alpha) < \infty$~(\cref{lem:normalizer}),
\aligns{
  \lim_{d\to\infty}
  \frac{I_d(t)}{H_{d,\alpha}} = \frac{B\paren{1-\frac{1}{\alpha}, 1+t}}{\alpha\zeta(\alpha)}.
}
As it is not easy to intuit the rate from the Beta function, 
we give an additional asymptotic equivalence for large $t$.
Using Stirling's formula, the Beta function behaves as 
\aligns{
  B\paren{1-\frac{1}{\alpha}, 1+t} \asym{t} \Gamma\paren{1-\frac{1}{\alpha}} \frac{1}{t^{1-\frac{1}{\alpha}}}.
}

\textbf{For $\alpha < 1$} we use the change of variable $z = tk^{-\alpha}$ to get 
\aligns{
  I_d(t) = \frac{1}{\alpha} t^{\frac{1}{\alpha}-1} \int_{td^{-\alpha}}^t z^{-\frac{1}{\alpha}} \paren{1-\frac{z}{t}}^{t} \dif{z}.
}
To have a well-defined integral, 
we need to introduce the scaling $t_d(\tau) = \tau d^\alpha$,
\aligns{
  I_d(\tau d^\alpha) = 
  \frac{1}{\alpha} 
  d^{1-\alpha}
  \tau^{\frac{1}{\alpha}-1}
  \int_{\tau}^{\tau d^\alpha} z^{-\frac{1}{\alpha}} 
  \paren{1-\frac{z}{\tau d^\alpha}}^{\tau d^\alpha}
  \dif{z}.
}
The factor of $d^{1-\alpha}$ will cancel out with the normalizer as $H_{d,\alpha} = \Theta(d^{1-\alpha})$ (\cref{lem:normalizer}).
The remaining integral should simplify for large $d$, as $\paren*{1-{z}/{\tau d^\alpha}}^{\tau d^\alpha} \approx e^{-z}$, 
and converge to 
\aligns{
  \lim_{d\to\infty}
  \tau^{\frac{1}{\alpha}-1}
  \int_{\tau}^{\tau d^\alpha} z^{-\frac{1}{\alpha}} 
  \paren{1-\frac{z}{\tau d^\alpha}}^{\tau d^\alpha}
  \dif{z}
  =
  \tau^{\frac{1}{\alpha}-1}
  \int_{\tau}^{\infty} z^{-\frac{1}{\alpha}} 
  e^{-z}
  \dif{z}
  =
  E_{\frac{1}{\alpha}}(\tau),
}
where $E_{p}$ is the generalized exponential integral.
To swap the limit and integral, we can verify that the dominated convergence theorem applies.
The integral can be written as
\aligns{
  \int_{\tau}^{\tau d^\alpha} 
  \!\!\!\!  z^{-\frac{1}{\alpha}} 
  \paren{1-\frac{z}{\tau d^\alpha}}^{\tau d^\alpha}
  \!\!\!\!\! 
  =
  \int_\tau^\infty 
  \!\!  
  a(z,d)
  \dif{z}
  \,\,\text{ where } \,\,
  a(z,d) \coloneq
    \ind{\{z\leq\tau d^\alpha\}}
    z^{-\frac{1}{\alpha}} 
    \paren{1-\frac{z}{\tau d^\alpha}}^{\tau d^\alpha}
    \!\!\!.
} 
The integrand $a(z, d)$ converges pointwise to $f(z) = z^{-\frac{1}{\alpha}}e^{-z}$
and is dominated by $f$ which is integrable as $\int_\tau^\infty f(z) = \tau^{1-\frac{1}{\alpha}}E_{\frac{1}{\alpha}}(\tau)$.
Combined with the fact that $H_{d,\alpha} \asym{d} d^{1-\alpha}/(1-\alpha)$, we get 
\aligns{
  \lim_{d\to\infty} \frac{I_d(\tau d^\alpha)}{H_{d,\alpha}}
  = 
  \frac{1-\alpha}{\alpha} E_{\frac{1}{\alpha}}(\tau).
}
To simplify for large $\tau$ and obtain 
\smash{$E_{\nicefrac{1}{\alpha}}(\tau) \asym{\tau} \nicefrac{e^{-\tau}}{\tau}$},
we use the fact that the generalized exponential integral $E_p(z)$ in decreasing in $p$, 
meaning that 
\smash{$E_{\floor{\nicefrac{1}{\alpha}}}(\tau) > E_{\nicefrac{1}{\alpha}}(\tau) > E_{\ceil{\nicefrac{1}{\alpha}}}(\tau)$},
and that for integer values of $p$ we have 
\smash{$\nicefrac{e^{-\tau}}{\tau+n} \leq E_n(\tau) \leq \nicefrac{e^{-\tau}}{\tau+n-1}$}
\citep[][\href{https://dlmf.nist.gov/8.19.ix}{Â§8.19(ix)}]{NIST:DLMF}.
Both bounds are asymptotically equivalent to $e^{-\tau}/(\tau+1)$.

\textbf{For $\alpha = 1$}
we use the change of variable $k = d^z$ or $z = \log_d(k)$ to get
\aligns{
  I_d(t) = \log(d) \int_0^1 \paren{1-d^{-z}}^{t} \dif{z}.
}
The normalizer scales as $H_{d,\alpha} \asym{d} \log(d)$ (\cref{lem:normalizer}) so only the integral remains.
To make meaningful progress, we introduce the scaling $t_d(\tau) = d^\tau$ for $\tau \in [0,1]$, 
\aligns{
  \frac{I_d(d^\tau)}{\log(d)} 
  = \int_0^1 \paren{1-\frac{d^{\tau-z}}{d^\tau}}^{d^\tau} \dif{z}.
}
As $d \to \infty$, the integrand converges to $0$ if $z \in (0,s)$ and to $1$ if $z \in (s,1)$,
and is dominated by $f(x) = 1$ so by the DCT we can swap the limit and integral to get
\aligns{
  \lim_{d\to\infty}
  \frac{I_d(d^\tau)}{H_{d,\alpha}} 
  &= 
  \lim_{d\to\infty}
  \int_0^1 
  \paren{1-\frac{d^{\tau-z}}{d^\tau}}^{d^\tau} 
  \dif{z}
  =
  \int_0^\tau 0 \dif{z}
  +
  \int_\tau^1 1 \dif{z}
  = 1-\tau.
  \tag*{\qedhere}
}
\end{proof}

\begin{lemma}[Approximation error is negligible]
\label{lem:approximation-error-negligible}
Let $\delta_d(t)$ be the upper bound on the approximation error 
derived in the \LemmaApproxError.
We have that 
\aligns{
  \lim_{d\to\infty} 
  \frac{
    \delta_d(t_d(\tau))
  }{
    I_d(t_d(\tau))
  }
  = 0 
  \text{ if } \alpha \leq 1,
  &&\text{ and }&&
  \lim_{\tau\to\infty} 
  \lim_{d\to\infty} 
  \frac{
    \delta_d(t)
  }{
    I_d(t)
  }
  = 0 
  \text{ if } \alpha > 1.
}
\end{lemma}
\begin{proof}
Recall that the bound approximation error $\delta$ in 
\LemmaApproxError{} is
\aligns{
  \abs{S_d(t) - I_d(t)}
  \leq \delta_d(t) 
  \quad \text{ where } \quad 
  \delta_d(t) 
  \coloneq
  \switch{
    \frac{1}{1+t}
    \paren{1-\frac{1}{1+t}}^t
    & \text{ if } 1+t \leq d^\alpha,
    \\
    \frac{1}{d^\alpha}\paren{
      1-\frac{1}{d^\alpha} 
    }^t
    & \text{ if } 1+t \geq d^\alpha.
  }
}

\textbf{For $\alpha > 1$,}
$t$ does not scale with $d$ so we are in the small $t$ regime, $1+t \leq d^\alpha$.
In this regime, 
\aligns{
  \delta_d(t) = \frac{1}{t+1}\paren{1-\frac{1}{t+1}}^{t} \leq \frac{1}{t+1}.
}
The error $\delta_d(t)$ does not vanish with $d$,
but it goes down as $O(1/t)$.
As the integral $I_d(t)$ is of order \smash{$\Theta(1/t^{1-\frac{1}{\alpha}})$},
the relative error is of order \smash{$O(1/t^\frac{1}{\alpha})$},
and vanishes for large~$t$.

\textbf{For $\alpha < 1$,}
we scale $t$ with $d$ as $t = \tau d^\alpha$.
Whether $t$ is small or large depends on $\tau$.
If $\tau < 1$, we are in the small $t$ regime as $1 + \tau d^\alpha \leq d^\alpha$ and
\aligns{
  \delta_d(\tau d^\alpha) = \frac{1}{\tau d^\alpha+1}\paren{1-\frac{1}{\tau d^\alpha+1}}^{\tau d^\alpha}
  \leq \frac{1}{\tau d^\alpha}.
}
If $\tau \geq 1$ we are in the large $t$ regime and 
\aligns{
  \delta_d(\tau d^\alpha) = \frac{1}{d^\alpha}\paren{1-\frac{1}{d^\alpha}}^{\tau d^\alpha}
  \leq \frac{1}{d^\alpha}.
}
In both cases $\lim_{d\to\infty} \delta_d(\tau d^\alpha) \to 0$
and the relative error also vanishes.

\textbf{For $\alpha = 1$}
we scale $t$ with $d$ as $t = d^\tau$ for $\tau \in [0,1]$.
Taking $d \to \infty$ puts us in the small $t$ regime,~$1 + t = 1 + d^\tau \leq d$.
In this regime,
\aligns{
  \delta_d(d^\tau) = 
  \frac{1}{d^\tau+1} \paren{1-\frac{1}{d^\tau+1}}^{d^\tau}
  \leq 
  \frac{1}{d^\tau},
}
which also vanishes with $d$.
\end{proof}

\clearpage
\section{Proofs for sign descent}
\label{apx:sign}

This section gives the derivation for the scaling
of time and the step-size for sign descent given in \cref{def:sign-scaling}
and the resulting asymptotic convergence rates of \cref{thm:sign-rate}.
Each result start from the relative loss defined as follows.
\begin{definition}[Normalized loss for sign descent]
\label{prop:sign-starting-point}
Let $\Loss_d(t,\eta)$ be the loss after with step-size $\eta$ as defined 
in \cref{prop:loss-of-sd},
and $\eta(T, \phi) = \nicefrac{1}{H_{d,\alpha}T\phi^\alpha}$
be the reparameterization of the step-size
derived from \cref{prop:step-size-range-for-sign}.
The relative loss after $T$ steps of the simplified sign descent dynamics 
on \cref{prob:bigram} with power-law frequencies as in \cref{ass:conditional-distribution}
is
\aligns{
  r_d(T,\phi) 
  \coloneqq 
  \frac{\Loss_d(T,\eta(T, \phi)) - \Loss_d^*}{\Loss_d(0) - \Loss_d^*}
  = 
  \frac{
    H_{\floor{\phi},2\alpha} 
    -2 H_{\floor{\phi},\alpha} \phi^{-\alpha}
    + \floor{\phi} \phi^{-2\alpha}
    + \frac{d-\floor{\phi}}{4T^2}\phi^{-2\alpha}
  }{
    H_{d,2\alpha}
  },
}
where $H_{n,p} = \sum_{k=1}^n k^{-p}$.
\end{definition}
\begin{proof}
Starting from \cref{prop:loss-of-sd}
and using the fact that, if $\phi \in [1,d]$, 
the number of components in the decreasing phase of 
the simplified sign descent dynamics
is $\floor{\phi}$, we expand the square and replacing the sums by $H_{n,p}$,
\aligns{
  r_d(T,\phi) 
  &= 
  \frac{
    \sum_{k=1}^{\floor{\phi}} \paren{k^{-\alpha}-\phi^{-\alpha}}^2 
    + \sum_{k=\floor{\phi}+1}^{d} \paren{\frac{1}{2T\phi^\alpha}}^2
  }{
    \sum_{k=1}^d k^{-2\alpha}
  },
  \\
  &= 
  \frac{
    \paren{
      \sum_{k=1}^{\floor{\phi}} k^{-2\alpha} -2k^{-\alpha}\phi^{-\alpha} + \phi^{-2\alpha}
    }
    + \frac{d-\floor{\phi}}{4T^2}\phi^{-2\alpha}
  }{
    \sum_{k=1}^d k^{-2\alpha}
  },
  \\
  &= 
  \frac{
    H_{\floor{\phi},2\alpha} 
    -2 H_{\floor{\phi},\alpha} \phi^{-\alpha}
    + \floor{\phi} \phi^{-2\alpha}
    + \frac{d-\floor{\phi}}{4T^2}\phi^{-2\alpha}
  }{
    H_{d,2\alpha}
  }.
  \tag*{\qedhere}
}
\end{proof}
Our rates are given 
for a choice of scaling of the step-size 
$\phi_d(\tau)$
and time $T_d(\tau)$, 
as 
\aligns{
  r(\tau) \coloneq \lim_{d\to\infty} r_d(T_d(\tau), \phi_d(\tau)).
}

\subsection{Scaling of sign descent for $\alpha = 1/2$}
\begin{proposition}
\label{prop:apx-sign-medium}
For the relative loss defined in \cref{prop:sign-starting-point},
if $\alpha = 1/2$, the scalings
\aligns{
  T_d(\tau) = \tfrac{1}{2} d^{\frac{1}{2}\tau},
  &&
  \phi_d(\tau) = d^{1-\tau},
}
are obtained by setting 
$\phi_d(\tau) = d^{x_*(\tau)}$ where
$x_*(\tau)$ is the solution to
\aligns{
  x_*(\tau) = \arg\min_{0<x\leq1} \lim_{d\to\infty}
  r_d(T_d(\tau), d^x).
}
These choices result in the scaling
$r(\tau) = 1-\tau$.
\end{proposition}
\begin{proof}
We start from the normalized loss given $\phi$,
\aligns{
  r_d(T,\phi) 
  &= 
  \frac{
    H_{\floor{\phi},1} 
    -2 H_{\floor{\phi},\frac{1}{2}} \phi^{-\frac{1}{2}}
    + \floor{\phi} \phi^{-1}
    + \frac{d}{4T^2}\phi^{-1}
    - \frac{1}{4T^2}\floor{\phi}\phi^{-1}
  }{
    H_{d,1}
  }.
}
Taking $4T^2 = d^\tau$ and $\phi = d^{1-\tau}$,
most terms vanish as $d \to \infty$ 
as $H_{n,\frac{1}{2}} \asym{} 2\sqrt{n}$, $H_{n,1} \asym{} \log(n)$,~and
\aligns{
  \frac{
    2 H_{\floor{d^{1-\tau}},\frac{1}{2}} d^{-\frac{1-\tau}{2}}
  }{
    H_{d,1}
  },
  \frac{
  \floor{d^{1-\tau}} d^{-(1-\tau)}
  }{
    H_{d,1}
  },
  \frac{
    1
  }{
    H_{d,1}
  },
  \frac{
  \floor{d^{1-\tau}}d^{-1}
  }{
    H_{d,1}
  } 
  \text{ are all }
  \Theta\paren{\frac{1}{\log(d)}}
  \text{ and converge to } 0.
  }
The first term is the only one remaining,
and gives the scaling
\aligns{
  \lim_{d\to\infty} 
  r_d(T(d,\tau), d^{x}) 
  = 
  \lim_{d\to\infty} 
  \frac{H_{\floor{d^{1-\tau}}, 1}}{H_{d,1}}
  = 
  \lim_{d\to\infty} 
  \switch{
    x & \text{ if } 1 - \tau \leq x,
    \\
    \infty & \text{ otherwise}.
  }
}
The optimum is at $x_*(\tau) = 1-\tau$ and gives 
$r(\tau) = \lim_{d\to\infty} r_d(T(d,\tau), d^{1-\tau}) = 1-\tau$.
\end{proof}

\clearpage
\subsection{Scaling of sign descent for $\alpha < 1/2$}
\begin{proposition}
\label{prop:apx-sign-small}
For the relative loss defined in \cref{prop:sign-starting-point},
if $\alpha < 1/2$, the scalings
\aligns{
  T_d(\tau) = \tau,
  &&
  \phi_d(\tau) = 
  \left\{\begin{array}{ll}
    d & \text{ if } \tau \leq \sqrt{\frac{1-c_1}{4c_2}},
    \\
    d \paren{c_1 + c_2 4\tau^2}^{-1}
    & \text{ otherwise},
  \end{array}
  \right.
}
where $c_1 = 1-\frac{1}{2\alpha}$ and $c_2 = \frac{\alpha}{\alpha-1}$,
are obtained by setting $\phi_d(\tau) = dx_*(\tau)$ where
\aligns{
  x_*(\tau) = \arg\min_{0<x\leq1} \lim_{d\to\infty} r_d(T_d(\tau), dx).
}
These choices result in the scaling
\aligns{
  r(\tau)
  =
  \left\{
    \begin{array}{ll}
      2\alpha c_2 & \text{ if }\tau \leq \sqrt{\frac{1-c_1}{4c_2}}
      \\
      \frac{
        (c_1 + c_2 4\tau^2)^{2\alpha}
      }{
        4\tau^2 
      }
      & \text{ otherwise}
    \end{array}
  \right.
  \asym{\tau} 
  c_2^{2\alpha}\frac{1}{(2\tau)^{2-4\alpha}}.
}

\end{proposition}
\begin{proof}
Substituting $\phi = dx$, taking the limit as $d \to \infty$, 
and using that $H_{d,p} \asym{} \frac{d^{1-p}}{1-p}$ for $p < 1$, 
define $f_\tau(x)$ as the limit of $r_d(\tau, dx)$ as $d$ grows,
\aligns{
  f_\tau(x) = \lim_{d\to\infty} r_d(\tau,dx)
  &= \lim_{d\to\infty} \frac{
    H_{\floor{dx},2\alpha} 
    -2 H_{\floor{dx},\alpha} (dx)^{-\alpha}
    + \floor{dx} (dx)^{-2\alpha}
    + \frac{d-\floor{dx}}{4\tau^2}(dx)^{-2\alpha}
  }{
    H_{d,2\alpha}
  },
  \\ 
  &=
  \frac{
    \frac{1}{1-2\alpha}x^{1-2\alpha}
    -2\frac{1}{1-\alpha}x^{1-2\alpha}
    + x^{1-2\alpha}
    + \frac{1}{4\tau^2}x^{-2\alpha}
    - \frac{1}{4\tau^2}x^{1-2\alpha}
  }{
    \frac{1}{1-2\alpha}
  }.
}
We will show that our choice of step-size 
corresponds to taking $r(\tau) = \min_{0<x\leq1} f_\tau(x)$.
Gathering terms,~$f_\tau(x)$ is proportional to the following polynomial
\aligns{
  f_\tau(x) \propto
  x^{1-2\alpha} \paren{
    1 + \frac{1}{1-2\alpha}
    -2\frac{1}{1-\alpha}
    -\frac{1}{4\tau^2}
  }
  + \frac{1}{4\tau^2}x^{-2\alpha},
}
which has a unique stationary point at 
\newcommand{\xstat}{x_{\mathrm{stat}}}
\aligns{
  \xstat(\tau) &= 
  \frac{2\alpha}{4\tau^2} \frac{1}{(1-2\alpha) \paren{
    1 + \frac{1}{1-2\alpha}
    -2\frac{1}{1-\alpha}
    -\frac{1}{4\tau^2}
  }}
  = \paren{1-\frac{1}{2\alpha} + \frac{\alpha}{1-\alpha} 4\tau^2}^{-1}.
}
If $\xstat(\tau) \not\in [0,1]$, 
$f_\tau(x)$ must be decreasing over $[0,1]$
as $\lim_{x\to0}f_\tau(x) = \infty$ and $r_\tau(1)$ is finite, 
and the minimum must be at $1$.
If the stationary point is in $(0,1]$, it must be the minimum.
This gives
\aligns{
  x_* = \arg\min_{0<x\leq1} f_\tau(x) = 
  \left\{\begin{array}{lll}
    x_{\mathrm{stat}} & \text{ if } 0 < x_{\mathrm{stat}} \leq 1,
    \\
    1 & \text{ otherwise}.
  \end{array}\right.
}
and $0 < \xstat(\tau) \leq 1$ is equivalent to $\tau \geq \frac{1}{2}\sqrt{\frac{1-\alpha}{2\alpha^2}}$.
If $\tau \geq \frac{1}{2}\sqrt{\frac{1-\alpha}{2\alpha^2}}$
and $x_*(\tau) = 1$, we get
\aligns{
  f_\tau(x_*(\tau))
  = 
  1 - 2\frac{1-2\alpha}{1-\alpha} + (1-2\alpha)
  = 2\frac{\alpha^2}{1-\alpha}.
}
If $\tau < \frac{1}{2}\sqrt{\frac{1-\alpha}{2\alpha^2}}$
and $x_*(\tau) = \paren{1-\frac{1}{2\alpha} + \frac{\alpha}{1-\alpha} 4\tau^2}^{-1}$
we get 
\aligns{
  f_\tau(x_*(\tau)) 
  &=
  \paren{1-2\alpha}
  \paren{
  x^{1-2\alpha} \paren{
    1 + \frac{1}{1-2\alpha}
    -2\frac{1}{1-\alpha}
    -\frac{1}{4\tau^2}
  }
  + \frac{1}{4\tau^2}x^{-2\alpha}
  },
  \\
  &= 
  \frac{
    \paren{1-\frac{1}{2\alpha}+\frac{\alpha}{1-\alpha}4\tau^2}^{2\alpha}    
  }{
    4\tau^2
  },
}
which can be simplified for large $\tau$ as 
$f_\tau(x_*(\tau)) 
  \asym{\tau}
  \paren{\frac{\alpha}{1-\alpha}}^{2\alpha}
  \frac{
    1
  }{
    \paren{2\tau}^{2-4\alpha}
  }
$.
\end{proof}

\clearpage
\subsection{Scaling of sign descent for $\alpha > 1/2$}

For $\alpha > 1/2$, 
the expression for the loss does not simply as $d \to \infty$.
The conditional frequencies decay fast, 
meaning that most of the loss comes from the few high-frequency words.
As a result, we cannot define the scaling of the step-size 
as the minimization problem for the optimal scaling in the limit $d \to \infty$.
Instead, we use the fact that the (normalized) loss can not converge to 0 
unless all components enter the oscillatory regime, 
at which point we can compute an optimal step-size.
\begin{proposition}
\label{prop:apx-sign-large}
For the relative loss defined in \cref{prop:sign-starting-point},
if $\alpha > 1/2$
and $4T^2 \geq \frac{d-1}{2^\alpha-1}$, the optimal-step size is given by 
\aligns{
  \phi_*(d, T) = \arg\min_{\phi} r_d(T,\phi)
  = \paren{1+\frac{d-1}{4T^2}}^{1/\alpha}.
}
This gives the following scaling 
for $\tau^2 > \nicefrac{1}{(2^\alpha-1)}$
\aligns{
  T_d(\tau) = \tau \tfrac{1}{2}\sqrt{d},
  && 
  \phi(\tau) = \paren{1+\frac{1}{\tau^2}}^{1/\alpha},
  &&
  r(\tau)
  =
  \frac{1}{\zeta(2\alpha)}\frac{1}{1+\tau^2}.
}
\end{proposition}
\begin{proof}
If $\phi \geq 2$, 
the normalized loss is lower-bounded by the error on the first two components,
\aligns{
  r_d(T,\eta(T, \phi)) 
  &=
  \frac{
    \sum_{k=1}^{\floor{\phi}} \paren{k^{-\alpha}-\phi^{-\alpha}}^2 
    + \sum_{k=\floor{\phi}+1}^{d} \paren{\frac{1}{2T\phi^\alpha}}^2
  }{
    H_{d,2\alpha}
  }.
}
This is lower-bounded by a constant $C > 0$ independently of $T$,
and implies that we cannot make progress by running longer unless $\phi < 2$.
If only the first component is oscillating, the optimal $\phi$ is 
\aligns{
  \phi_*(d,T) = 
  \argmin_\phi
  r_d(T,\eta(T, \phi))
  = 
  \argmin_\phi
  (1-\phi^{-\alpha})^2 
  + \frac{d-1}{4T^2}\phi^{-2\alpha}
  = \paren{1 +\frac{d-1}{4T^2}}^{1/\alpha}.
}
To be consistent with only having two components oscillating, this requires
$\phi_*(d, T) \leq 2$, 
giving the constraint that this only holds when 
\smash{$\paren*{1 +\frac{d-1}{4T^2}}^{1/\alpha} \leq 2$} or 
\smash{$4T^2 \geq \frac{d-1}{2^\alpha-1}$}.
Taking the scaling~\smash{$4T_d(\tau)^2 = \tau^2 d$} gives the limit 
\aligns{
  \phi(\tau) = \lim_{d\to\infty} \phi_*(d, T_d(\tau)) = \paren{1 + \frac{1}{\tau^2}}^{1/\alpha}
  \quad \text{ if } \quad 
  \tau^2 > \frac{1}{2^\alpha-1},
}
and the asymptotic loss 
\aligns{
  \lim_{d\to\infty} r_d(T_d(\tau, d), \phi(\tau)) 
  &= 
  \frac{
    \paren{1-\phi(\tau)^{-\alpha}}^2 
    + \frac{1}{\tau^2}\phi(\tau)^{-2\alpha}
  }{
    \zeta(2\alpha)
  },
  \\
  &=
  \frac{
    \paren{1-\paren{1+\frac{1}{\tau^2}}^{-1}}^2 
    + \frac{1}{\tau^2}\paren{1+\frac{1}{\tau^2}}^{-2}
  }{
    \zeta(2\alpha)
  }
  =
  \frac{1}{1+\tau^2}\frac{1}{\zeta(2\alpha)},
}
where $H_{d,2\alpha} \asym{d} \zeta(2\alpha)$, the Riemann zeta function.
\end{proof}

\cref{prop:apx-sign-large} and \cref{thm:sign-rate}
only gives guarantees for the regime $\tau^2 > \nicefrac{1}{(2^\alpha-1)}$.
The extension of the scalings 
to the regime $\tau^2 \leq \nicefrac{1}{(2^\alpha-1)}$
was decided arbitrarily to fit empirical data. 
To fit the empirical the empirical data 
when both $\tau$ and $\alpha$ are small ($\alpha \leq 1$),
the asymptotic scaling presented in \cref{thm:sign-rate}
uses the following step-size scaling
\aligns{
  \tilde \phi(\tau) 
  = \left\{\begin{array}{ll}
    \paren{1+\frac{1}{\tau^2}} & \text{ if } \tau^2 < (2^\alpha-1)^{-1}  
    \text{ and } \alpha < 1, \\
    \paren{1+\frac{1}{\tau^2}}^{1/\alpha} & \text{ otherwise},
  \end{array}\right.
  \quad \text{ instead of } \quad 
  = \paren{1+\frac{1}{\tau^2}}.
}
and the following approximation for the loss,
\aligns{
  r_d(T_d(\tau, d), \phi(\tau)) 
  \asym{\tau,d}
  \frac{1}{1+\zeta(2\alpha)\tau^2}
  \quad \text{ instead of } \quad 
  \frac{1}{1+\tau^2}\frac{1}{\zeta(2\alpha)}.
}
Both expressions are asymptotically equivalent as $d\to \infty$ and $\tau \to \infty$, 
but the above proposals (given in \cref{def:sign-scaling}) 
fit the observed best step-size and loss scalings better.

\clearpage 
\section{Scaling as a function of desired relative error}
\label{apx:additional_details}

This section derives the results of \cref{thm:informal-gd},
showing how the number of iterations should scale as a function of $d$ and the desired relative error $\varepsilon$.
The main theorems (\cref{thm:gradient-descent,thm:sign-rate})
give the error $\varepsilon$ as a function of the rescaled time $t_d(\tau)$, 
this section gives the inversion.

\begin{proposition}[Formal version of \cref{thm:informal-gd} for gradient descent]
To reach a relative loss of $\varepsilon$ using gradient descent (\cref{thm:gradient-descent}), $t$ needs to scale with $d$ as follows.
\aligns{
  \text{ If } \alpha < 1, && t_d(\varepsilon) &= d^\alpha \tilde\Theta(\log(1/\varepsilon)),
  \\
  \text{ if } \alpha = 1, && t_d(\varepsilon) &= d^{1-\varepsilon},
  \\
  \text{ if } \alpha > 1, && t_d(\varepsilon) &= \paren{C/\varepsilon}^{\frac{\alpha}{\alpha-1}}
  \text{ where } C = \tfrac{\Gamma(1-\frac{1}{\alpha})}{\alpha\zeta(2\alpha)}.
}
With those scalings, 
we have that $\lim_{d\to\infty} r_d(t_d(\varepsilon)) = \varepsilon$
where $r_d(t)$ is  the relative loss defined in \cref{prop:simplified-dynamics-gd}, 
or $\lim_{\varepsilon\to0} \lim_{d\to\infty} r_d(t_d(\varepsilon))/\varepsilon = 1$
in the case $\alpha > 1$.
\end{proposition}
\begin{proof}
\textbf{For $\alpha = 1$, } \cref{thm:gradient-descent} shows that 
\aligns{
  r(\tau) \coloneq 
  \lim_{d\to\infty} \frac{\Loss_d(t_d(\tau)) - \Loss_d^*}{\Loss_d(0) - \Loss_d^*} = 1-\tau
  \text{ with } t_d(\tau) = d^\tau
  \quad \implies \quad
  \lim_{d\to\infty} \frac{\Loss_d(d^{1-\varepsilon}) - \Loss_d^*}{\Loss_d(0) - \Loss_d^*}
  =
  \varepsilon.
}

\textbf{For $\alpha > 1$,}
we know that 
$r(\tau) \asym{\tau} 
\nicefrac{C}{\tau^{1-\frac{1}{\alpha}} }$.
Letting $\tau(\varepsilon) = \paren{C/\varepsilon}^{\frac{\alpha}{\alpha-1}}$ 
gives that $r(1/\varepsilon) \asym{1/\varepsilon} \varepsilon$ as 
\aligns{
  \lim_{\tau\to\infty} \lim_{d\to\infty}
  r_d(\tau) \frac{1}{C} \tau^{\frac{1-\alpha}{\alpha}}
  = 1
  \implies
  \lim_{\varepsilon\to0} \lim_{d\to\infty}
  r_d(1/\varepsilon) \frac{1}{C} \paren{\paren{C/\varepsilon}^{\frac{\alpha}{\alpha-1}}}^{\frac{1-\alpha}{\alpha}}
  =
  \lim_{\varepsilon\to0} \lim_{d\to\infty}
  r_d(1/\varepsilon) \varepsilon = 1.
}

\textbf{For $\alpha < 1$}
We have that $\lim_{d\to\infty} r_d(\frac{1}{2}\tau d^\alpha) = \frac{1-\alpha}{\alpha} E_{\frac{1}{\alpha}}(\tau)$.
The exact solution is
\aligns{
  \tau(\varepsilon) = E_{\frac{1}{\alpha}}^{-1}\paren{\frac{\alpha}{1-\alpha}\varepsilon},
}
where $E_{{1}/{\alpha}}^{-1}$ is the inverse of the generalized exponential integral.
To get an idea of its growth, we show \smash{$\tau(\varepsilon) = \tilde\Theta(\log(1/\varepsilon))$}
by showing that for small enough $\varepsilon$,
\aligns{
  \tau_-(\epsilon)
  \leq 
  \tau(\varepsilon)
  \leq 
  \tau_+(\varepsilon)
  \quad &\text{if} \quad
  \varepsilon \leq 
  \min\braces{e^{2-\floor{\frac{1}{\alpha}}}, \frac{W(6)}{6}} \frac{1-\alpha}{\alpha},
  \\
  \text{ where }
  \tau_+(\varepsilon)
  &= 
  \log\paren{\frac{1-\alpha}{\alpha}\frac{1}{\varepsilon}},
  \\
  \tau_-(\varepsilon)
  &= 
  \log\paren{\frac{1-\alpha}{\alpha}\frac{1}{\varepsilon}}
  -\!\log\!\log\paren{\frac{1-\alpha}{\alpha}\frac{1}{\varepsilon}}
  - \ceil{\frac{1}{\alpha}}.
}
To show that $\tau(\varepsilon)$ must be withing that interval, 
we first derive an interval on $\varepsilon$ as a function of $\tau$
using the same bounds on the generalized exponential integral 
as in \cref{lem:asymptotic-integral},
\aligns{
  E_{\ceil{\frac{1}{\alpha}}}(\tau)
  \leq 
  E_{\frac{1}{\alpha}}(\tau)
  \leq
  E_{\floor{\frac{1}{\alpha}}}(\tau)
  &&\text{ and }&&
  \frac{e^{-\tau}}{\tau+n}
  \leq 
  E_{n}(\tau)
  \leq
  \frac{e^{-\tau}}{\tau+n-1}
  \text{ for } n \in \mathbb{N}.
}
Combining both bounds gives the range
\aligns{
  a(\tau) \coloneqq 
  \frac{1-\alpha}{\alpha}\frac{e^{-\tau}}{\tau+\ceil{\frac{1}{\alpha}}}
  \leq
  \varepsilon 
  \leq
  \frac{1-\alpha}{\alpha}\frac{e^{-\tau}}{\tau+\floor{\frac{1}{\alpha}}-1}
  \eqcolon b(\tau).
}
We will show that using $\tau_+$ excludes $\varepsilon$ from the interval from the right, $b(\tau_+(\varepsilon)) < \tau$, 
and that $\tau_-$ excludes $\varepsilon$ from the left, $\varepsilon < a(\tau_-(\varepsilon))$,
if $\varepsilon$ is sufficiently small.
As both bounds decrease with $\tau$, this gives that $\tau(\varepsilon)$ must increase 
grow faster than $\tau_-(\varepsilon)$ but slower than $\tau_+(\varepsilon)$.

For the upper bound, 
we show that $b(\tau_+(\varepsilon)) < \varepsilon$
is $\varepsilon$ is small enough.
\aligns{
  b(\tau_+(\varepsilon)) 
  &= 
  \frac{1-\alpha}{\alpha}
  \frac{
    e^{\log\paren{\frac{\alpha}{1-\alpha}\varepsilon}}
  }{
    \log\paren{\frac{1-\alpha}{\alpha}\frac{1}{\varepsilon}} + \floor{\frac{1}{\alpha}}-1
  }
  = 
  \frac{\varepsilon}{
    \log\paren{\frac{1-\alpha}{\alpha}\frac{1}{\varepsilon}} + \floor{\frac{1}{\alpha}}-1
  }.
}
This bound is smaller than $\varepsilon$ for small enough $\varepsilon$, 
as 
\aligns{
  \log\paren{\frac{1-\alpha}{\alpha}\frac{1}{\varepsilon}} + \floor{\frac{1}{\alpha}}-1 > 1
  \equiv
  \frac{1-\alpha}{\alpha}\frac{1}{\varepsilon} > e^{2 - \floor{\frac{1}{\alpha}}}
  \equiv
  \varepsilon 
  < 
  e^{2 - \floor{\frac{1}{\alpha}}}
  \frac{1-\alpha}{\alpha}.
}

For the lower-bound, 
we show that $a(\tau_-(\varepsilon)) > \varepsilon$
if $\varepsilon$ is small enough.
We have
\aligns{
  a(\tau_-(\varepsilon))
  = 
  \varepsilon
  \frac{
    \log\paren{\frac{1-\alpha}{\alpha}\frac{1}{\varepsilon}}
    e^{\ceil{\frac{1}{\alpha}}}
  }{
    \log\paren{\frac{1-\alpha}{\alpha}\frac{1}{\varepsilon}}
    - \log\log\paren{\frac{1-\alpha}{\alpha}\frac{1}{\varepsilon}}
  },
}
and need to show that for a small enough $\varepsilon$,
\aligns{
  \frac{
    \log\paren{\frac{1-\alpha}{\alpha}\frac{1}{\varepsilon}}
    e^{\ceil{\frac{1}{\alpha}}}
  }{
    \log\paren{\frac{1-\alpha}{\alpha}\frac{1}{\varepsilon}}
    - \log\log\paren{\frac{1-\alpha}{\alpha}\frac{1}{\varepsilon}}
  } 
  \geq
  \frac{
    7\log\paren{y}
  }{
    \log\paren{y} - \log\log\paren{y}
  } 
  >
  1,
  \text{ where }
  y = \frac{1-\alpha}{\alpha}\frac{1}{\varepsilon},
}
where we used that $\ceil{1/\alpha} \geq 2$ and $e^2 \geq 7$.
We get that 
$7 \log(y) > \log(y) - \log\log(y)
\equiv
y^{6}\log(y) > 1$,
holds for all $y > C$ where $C = e^{W(6)/6} = 1.2696...$
where $W$ is the Lambert $W$ function~\citep[][\href{https://dlmf.nist.gov/4.13}{\S4.13}]{NIST:DLMF}
or in terms of $\varepsilon$,
\aligns{
  \frac{1-\alpha}{\alpha}\frac{1}{\varepsilon} > C
  \equiv
  \varepsilon < \frac{1}{C} \frac{1-\alpha}{\alpha}.
  \tag*{\qedhere}
}
\end{proof}

\begin{proposition}[Formal version of \cref{thm:informal-gd} for sign descent]
To reach a relative loss of $\varepsilon$ using sign descent (\cref{thm:sign-rate}), $t$ needs to scale with $d$ as follows.
\aligns{
  \text{ If } \alpha < 1/2, && t_d(\varepsilon) &= \frac{1}{2} (c_2)^{\frac{\alpha}{1-2\alpha}} \paren{\nicefrac{1}{\varepsilon}}^{\frac{1}{2-4\alpha}},    \\
  \text{ if } \alpha = 1/2, && t_d(\varepsilon) &= \frac{1}{2}d^{\frac{1}{2}(1-\varepsilon)},         \\
  \text{ if } \alpha > 1/2, && t_d(\varepsilon) &= \frac{1}{2}\bigg(\frac{d \paren{\nicefrac{1}{\varepsilon}-1}}{\zeta(2\alpha)} \bigg)^{1/2}.
}
With those scalings, we have that $\lim_{d\to\infty} r_d(t_d(\varepsilon), \phi_d(\varepsilon)) = \varepsilon$
for some choice of step-size $\phi_d(\varepsilon)$
where $r_d(t, \phi)$ is the relative loss defined in \cref{prop:simplified-dynamics-sd}.
\end{proposition}
\begin{proof}
\textbf{For $\alpha = 1/2$,} \cref{prop:apx-sign-medium} shows 
the following asymptotic rate with the scaling 
$T_d(\tau) = \frac{1}{2}d^{\frac{1}{2}\tau}$
and the proper choice of step-size $\phi_d(\tau)$.
Substituting $\tau(\varepsilon) = 1-\varepsilon$ gives 
\aligns{
  \lim_{d\to\infty} r_d(T_d(\tau), \phi_d(\tau))
  = 1-\tau 
  \implies 
  \lim_{d\to\infty} r_d\paren{T_d\paren{1-\varepsilon}, \phi_d(1-\varepsilon)}
  = 1-\varepsilon.
}
Using that $T_d(\paren{1-\varepsilon}) = \frac{1}{2}d^{\frac{1}{2}(1-\varepsilon)}$
recovers the scaling in $d$.

\textbf{For $\alpha < 1/2$,}
\cref{prop:apx-sign-small}
shows the following asymptotic rate with the proper choice of step-size $\phi_d(\tau)$, 
with no scaling in $d$ for time, $T_d(\tau) = \tau$.
Substituting $\tau(\varepsilon) = \frac{1}{2}(\nicefrac{c_2^{2\alpha}}{\varepsilon})^{\frac{1}{2-4\alpha}}$ yields
\aligns{
  &\lim_{\tau\to\infty}
  \lim_{d\to\infty}
  r_d(\tau, \phi_d(\tau)) 
  \frac{1}{c_2^{2\alpha}} (2\tau)^{2-4\alpha}
  = 1,
  \\
  \implies 
  &\lim_{\varepsilon\to0}
  \lim_{d\to\infty}
  r_d(\tau(\varepsilon), \phi_d(\tau(\varepsilon))) 
  \frac{1}{c_2^{2\alpha}} \paren{2\tau(\varepsilon)}^{2-4\alpha}
  =
  \lim_{\varepsilon\to0}
  \lim_{d\to\infty}
  r_d(\tau(\varepsilon), \phi_d(\tau(\varepsilon))) 
  \frac{1}{\varepsilon}
  = 1.
}
Using that $T_d(\tau) = \tau$ finishes the proof.

\textbf{For $\alpha > 1/2$,}
\cref{prop:apx-sign-large} shows
the following asymptotic rate with the scaling 
$T_d(\tau) = \tau \frac{1}{2}{d}^{1/2}$
and the proper choice of step-size $\phi_d(\tau)$, 
which is valid for $\tau^2 > \nicefrac{1}{(2^\alpha-1)}$.
Substituting the scaling~\smash{$\tau(\varepsilon) = \paren*{\nicefrac{1}{\zeta(2\alpha)}\paren{\nicefrac{1}{\varepsilon}-1}}^{1/2}$},
\aligns{
  &
  \lim_{\tau\to\infty}
  \lim_{d\to\infty}
  r_d(\tau, \phi_d(\tau)) 
  (1+\zeta(2\alpha)\tau^2)
  = 1,
  \\
  \implies 
  &\lim_{\tau\to\infty}
  \lim_{d\to\infty}
  r_d(\tau(\varepsilon), \phi_d(\tau(\varepsilon))) 
  (1+\zeta(2\alpha)\tau(\varepsilon)^2)
  =
  \lim_{\tau\to\infty}
  \lim_{d\to\infty}
  r_d(\tau(\varepsilon), \phi_d(\tau(\varepsilon))) 
  \frac{1}{\varepsilon}
  = 1. 
}
Using that $T_d(\tau(\varepsilon)) = \tau(\varepsilon)\frac{1}{2}d^\frac{1}{2}$ finishes the proof.
\end{proof}

\end{document}